\documentclass[journal]{IEEEtran}

\usepackage{amsmath}
\usepackage{amsthm}
\usepackage{mathtools}
\usepackage{algorithm}  
\usepackage{algorithmicx}  
\usepackage{algpseudocode}  
\usepackage{bigstrut}
\usepackage{multirow}
\usepackage{subfigure}
\usepackage{booktabs} 
\usepackage{graphicx}
\usepackage{makecell}
\usepackage{xcolor}

\usepackage{amsmath,amsfonts,bm}

\def\eqref#1{equation~\ref{#1}}

\def\1{\bm{1}}

\def\vd{{\bm{d}}}
\def\ve{{\bm{e}}}

\def\vh{{\bm{h}}}

\def\vu{{\bm{u}}}

\def\vx{{\bm{x}}}

\def\mA{{\bm{A}}}
\def\mB{{\bm{B}}}
\def\mC{{\bm{C}}}
\def\mD{{\bm{D}}}
\def\mE{{\bm{E}}}

\def\mH{{\bm{H}}}
\def\mI{{\bm{I}}}

\def\mU{{\bm{U}}}

\def\mW{{\bm{W}}}
\def\mX{{\bm{X}}}
\def\mY{{\bm{Y}}}
\def\mZ{{\bm{Z}}}

\DeclareMathAlphabet{\mathsfit}{\encodingdefault}{\sfdefault}{m}{sl}
\SetMathAlphabet{\mathsfit}{bold}{\encodingdefault}{\sfdefault}{bx}{n}

\def\gE{{\mathcal{E}}}

\def\gG{{\mathcal{G}}}

\def\gL{{\mathcal{L}}}
\def\gM{{\mathcal{M}}}

\def\sV{{\mathbb{V}}}

\newcommand{\R}{\mathbb{R}}

\newtheorem{theorem}{Theorem}

\newtheorem{definition}{Definition}

\ifCLASSINFOpdf
\else
\fi
\begin{document}
%
\title{Structure-Aware DropEdge Towards Deep Graph Convolutional Networks}

\author{Jiaqi Han, Wenbing Huang$^\ast$, Yu Rong, Tingyang Xu, Fuchun Sun, Junzhou Huang
\IEEEcompsocitemizethanks{\IEEEcompsocthanksitem Jiaqi Han and Fuchun Sun are with the Department
of Computer Science and Technology, Tsinghua University, Beijing, China.\protect
\IEEEcompsocthanksitem Wenbing Huang is with Gaoling School of Artificial Intelligence, Renmin University of China;  Beijing Key Laboratory of Big Data Management and Analysis Methods, Beijing, China.\protect
\IEEEcompsocthanksitem Yu Rong and Tingyang Xu are with Tencent AI Lab, Shenzhen, China.\protect
\IEEEcompsocthanksitem Junzhou Huang is with the Department of Computer Science and Engineering, the University of Texas at Arlington. \protect
\IEEEcompsocthanksitem Jiaqi Han and Yu Rong contribute to this work equally. \protect
\IEEEcompsocthanksitem Corresponding author: Wenbing Huang. \protect
}
\thanks{Manuscript received April 19, 2005; revised August 26, 2015.}}

%
%

\markboth{Journal of \LaTeX\ Class Files,~Vol.~14, No.~8, August~2015}%
{Shell \MakeLowercase{\textit{et al.}}: Bare Demo of IEEEtran.cls for IEEE Journals}
%



\maketitle

\begin{abstract}
  It has been discovered that Graph Convolutional Networks (GCNs) encounter a remarkable drop in performance when multiple layers are piled up. The main factor that accounts for why deep GCNs fail lies in \emph{over-smoothing}, which isolates the network output from the input with the increase of network depth, weakening expressivity and trainability. In this paper, we start by investigating refined measures upon DropEdge---an existing simple yet effective technique to relieve over-smoothing. We term our method as DropEdge++ for its two structure-aware samplers in contrast to DropEdge:
  \emph{layer-dependent sampler} and \emph{feature-dependent sampler}. Regarding the layer-dependent sampler, we interestingly find that increasingly sampling edges from the bottom layer yields superior performance than the decreasing counterpart as well as DropEdge. We theoretically reveal this phenomenon with Mean-Edge-Number (MEN), a metric closely related to over-smoothing. For the feature-dependent sampler, we associate the edge sampling probability with the feature similarity of node pairs, and prove that it further correlates the convergence subspace of the output layer with the input features.  Extensive experiments on several node classification benchmarks, including both full- and semi- supervised tasks, illustrate the efficacy of DropEdge++ and its compatibility with a variety of backbones by achieving generally better performance over DropEdge and the no-drop version. 
\end{abstract}

\begin{IEEEkeywords}
Graph Convolutional Network, DropEdge++, Layer-Dependent sampler, Feature-Dependent sampler, Mean-Edge-Number.
\end{IEEEkeywords}

\definecolor{DarkBlue}{rgb}{0,0.08,1}
\newcommand{\revise}[1]{\textcolor{DarkBlue}{#1}}

%
\IEEEpeerreviewmaketitle

\section{Introduction}
\label{sec:introduction}
%
%
%
%
\IEEEPARstart{G}{raph} Convolutional Networks (GCNs)~\cite{Kipf2017,xu2018representation,zhang2021learning,xu2022deep,chen2022graph,qian2022quantifying,gong2022self} have recently dominated the analysis of graph structures on a variety of tasks, \emph{e.g.} node classification~\cite{Kipf2017,zhang2021learning,gong2022self}. Against those based on graph kernels~\cite{fouss2006experimental}, graph embedding~\cite{perozzi2014deepwalk,grover2016node2vec}, or matrix factorization and optimization~\cite{chen2022graph,liu2023a,liu2022symmetry,liu2021convergence}, the crucial benefit of using GCNs stems from that it is capable of representing both graph topology and node features in an end-to-end training manner. Whilst much progress has been made since~\cite{Kipf2017}, stacking more layers could be the most natural way to deliver a more expressive model that is expected to capture wider neighborhood proximity of each node. However, the performance of deep GCNs (even equipped with residual connections~\cite{he2016deep}) is much worse than their shallow variants on node classification~\cite{Kipf2017}. 


The most acknowledged answer to why deep GCNs fail is \emph{over-smoothing}. As firstly proposed by~\cite{Li2018}, over-smoothing drives the output of GCNs towards a stationary point (more precisely, a low-dimension subspace as proved by~\cite{Oono2020}) that contains limited topology information (\emph{i.e.} node degree) regardless of input features with the growing number of layers stacked. 
This property, unavoidably, causes gradient vanishing and hinders the training of deep GCNs. There have been developed various methods to alleviate such issue, including the architecture-based methods~\cite{xu2018representation,klicpera2018predict,chen2020simple} and the operation-based methods~\cite{chen2019measuring,rong2020dropedge,Zhao2020,zhou2020towards}. We specify the introduction of these methods in \textsection~\ref{sec:rw}.

Among the techniques proposed so far, DropEdge~\cite{rong2020dropedge}
uniformly drops out a certain number of edges at each training iteration, which is demonstrated to relieve over-smoothing theoretically and practically, and is able to enhance the performance of various kinds of deep GCNs. Yet and still, the improvement of deep GCNs by DropEdge over shallow models is marginal, leaving room to more powerful development. Our key motivation is the identification of the main bottleneck of DropEdge---its existing sampling strategy is unable to take full advantage of the structural information among the network design and data distribution: 1) the sampling of the adjacency matrix is shared across each layer (or sampled i.i.d. for the layer-wise version~\cite{rong2020dropedge}), hence the relation of sampling behavior between consecutive layers is never characterized; 2) it will neglect the different roles of different edges by sampling each edge uniformly.

In this paper, we put a step further and propose \emph{DropEdge++}, an enhanced version of DropEdge. DropEdge++ is composed of two structure-aware enhancements: \emph{layer-dependent sampler} (inter-layer) and \emph{feature-dependent sampler} (intra-layer). The layer-dependent sampler assumes the inductive bias that the adjacency matrix of the current layer is created by dropping edges conditional on the upper layer: \emph{the Layer-Increasingly-Dependent (LID) style}, or conditional on the lower layer: \emph{the Layer-Decreasingly-Dependent (LDD) style}. Interestingly, our experiments reveal that LID boosts the training of deep GCNs much better than LDD and DropEdge. This result implies that message passing over nodes in bottom layers are more vulnerable to over-smoothing and the edges in-between should be removed more than those of the top layers. We define a novel concept of Mean-Edge-Number (MEN) and show that LID derives smaller MEN than all others. The detailed explanations are provided in \textsection~\ref{sec:ld}.

Our second contribution---Feature-Dependent (FD) sampler leverages a certain kernel to measure the feature similarity of each adjacent node pair, and the normalized kernel value over all edges is then assigned as the edge sampling rate. As we will draw in this paper, our FD not only inherits all advantages of DropEdge on alleviating over-smoothing but also correlates the output of deep GCNs with the input features even when the number of layers approaches the infinity. More details are presented in \textsection~\ref{sec:FD}. 

In summary, this work makes the following contributions:
\begin{itemize}
    \item We excavate how the relations between adjacency matrices sampled in consecutive layers are connected to over-smoothing by introducing the layer-dependent sampler. We discover that increasingly sampling edges from the bottom layer yields superior performance than the decreasing counterpart and DropEdge. We reveal this phenomenon with Mean-Edge-Number (MEN), a metric we have found closely related to over-smoothing, both theoretically and empirically.
    \item As opposed to the uniform distribution in DropEdge, we propose feature-dependent sampler that takes into consideration different roles of different edges by adaptively sampling edges based on the feature similarity of the nodes connected. We also derive that the feature-dependent sampler relieves over-smoothing by enriching the convergence subspace with node feature information.
    \item We integrate these two samplers in a unified framework dubbed DropEdge++ that can be easily incorporated into existing GCNs and their architecture-based refinements~\cite{xu2018representation,klicpera2018predict,chen2020simple}. Experiments on six node classification benchmarks in both semi- and full-supervised fashion demonstrate the superiority of DropEdge++ in relieving over-smoothing, showing it remarkably outperforms DropEdge and the vanilla models without sampling.
\end{itemize}

\section{Related Work}
\label{sec:rw}

\textbf{Graph Neural Networks}
There have been diverse attempts with graph neural networks in literature~\cite{wu2021a}. The first prominent research \cite{bruna2013spectral} builds the seminal graph convolutional model in both the spatial and spectral views. Following~\cite{bruna2013spectral}, many works \cite{Kipf2017,zhang2021learning,xu2022deep,gong2022self,defferrard2016convolutional,henaff2015deep,Li2018a,Levie2017,he2022parallelly} have been done for performance enhancement and/or efficiency obtainment. Among them, the work by~\cite{Kipf2017} gives an efficient approximation to conduct the spectral operation in the spatial domain by aggregating the information from neighbor nodes. Later, GAT~\cite{DBLP:journals/corr/abs-1710-10903} introduces the attention mechanism to extend the learnability of message aggregation. KerGNN~\cite{feng2022kergnns} integrates graph kernels into the message passing process of GNNs, by adopting trainable hidden graphs as graph filters which are combined with subgraphs to update node embeddings using graph kernels. Although KerGNN also aims at enhancing the expressivity of GNNs similar to our DropEdge++, the fundamental goal of our paper is different and it is to investigate how to address the over-smoothing issue caused by the increase of depth. Notably, we also apply kernels for edge sampling in the feature-dependent sampler, but the kernels here are used in a different way compared to KerGNN.
Recently, to reduce the computational overhead of GCNs, various node-sampling-based methods are developed~\cite{hamilton2017inductive,Monti2017,niepert2016learning,Gao2018}, including the node-wise sampling~\cite{hamilton2017inductive}, layer-wise sampling~\cite{chen2018fastgcn,Huang2018}, and subgraph-wise sampling~\cite{zeng2019graphsaint}.

\textbf{Relieving Over-Smoothing}
It has been revealed that GCNs encounter a significant drop in performance with the depth beyond two~\cite{Kipf2017}. As first introduced in \cite{Li2018} and recently generalized by \cite{Oono2020}, over-smoothing stands in the way of building deep GCNs that the output converges to a subspace which is irrelevant to the input features. On aware of this, continuous efforts have been made among which the architecture-based and operation-based methods are two popular genres. The architecture-based methods refine deep GCNs by adding dense skip connections from intermediate layers to the output (\emph{i.e.} JKNet by~\cite{xu2018representation}), or reversely from the input layer to intermediate layers (\emph{i.e.} APPNP by~\cite{klicpera2018predict}). By this means, the activations in lower layers are preserved, thus capable of relieving over-smoothing. \cite{chen2020simple} founds that additionally conducting non-linear activation and identity mapping for each layer on APPNP further boosts the performance. 
On the other hand, the operation-based approaches develop certain operations upon the generic GCNs to prevent over-smoothing, such as the activation normalization~\cite{Zhao2020,zhou2020towards} and dropping edges~\cite{rong2020dropedge,chen2019measuring}. Our method, belonging to the operation-based family, delivers two enhancements of DropEdge and can be used along with the architecture-based methods.

\begin{figure*}[ht!]
\centering
\includegraphics[width=0.70\textwidth]{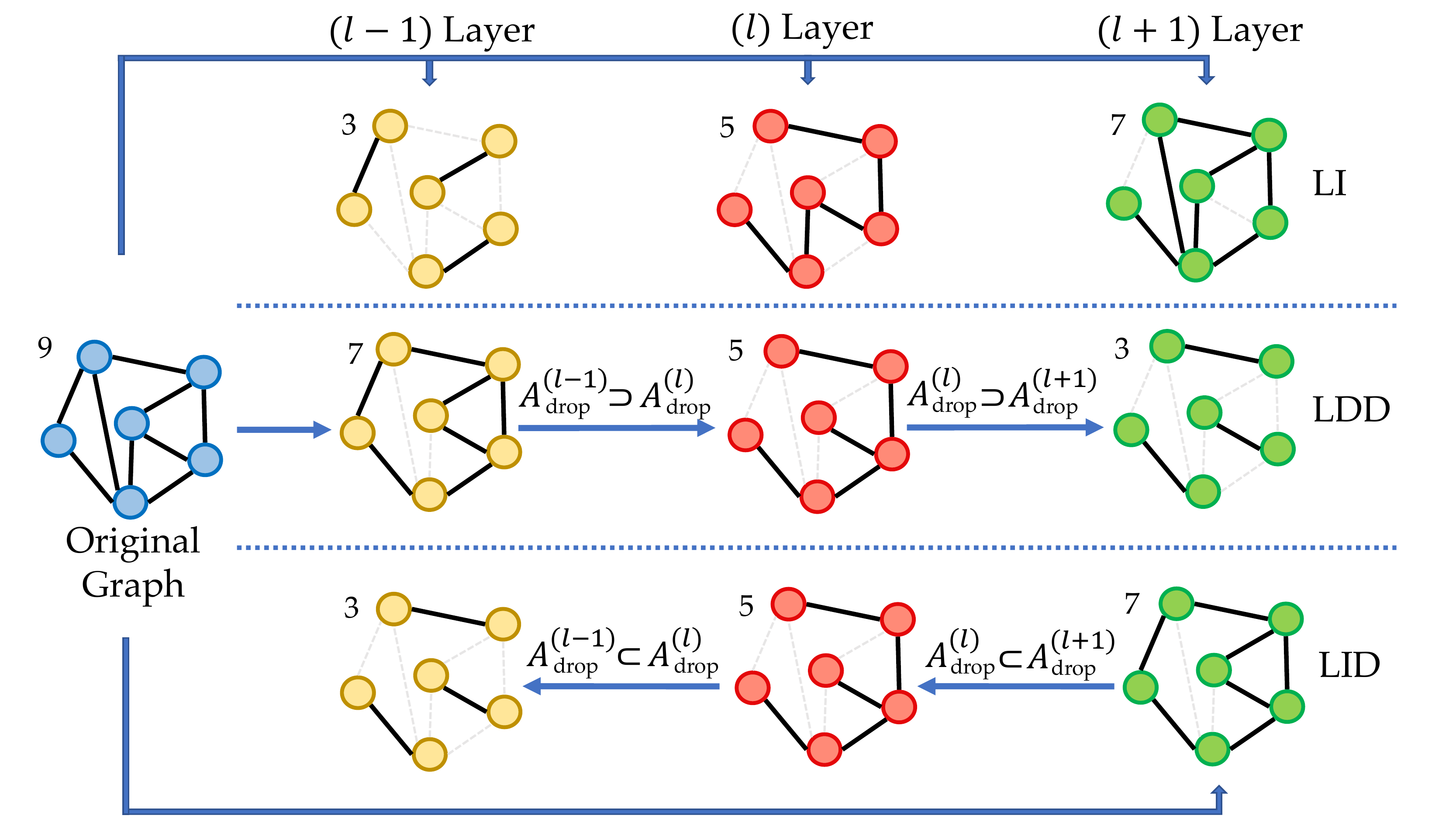}
\vskip -0.12 in
\caption{LI, LDD and LID. Solid lines are the preserved edges, and dashed lines are dropped. The number on the top left is $|\mA_{\text{drop}}^{(i)}|$ for the current layer $i$.
}
\label{fig.layerwise_samplers}
\vskip -0.1 in
\end{figure*}

\section{Preliminaries}

\textbf{Notations.}
We denote by $\gG=(\sV, \mathcal{E})$ the input graph of size $N$ with nodes $v_i\in\sV$ and edges $(v_i, v_j)\in\mathcal{E}$, by $\mX=\{\vx_1,\cdots,\vx_N\}\in\R^{N\times C}$ the node features, and by $\mA\in\R^{N\times N}$ the adjacency matrix which associates edge $(v_i, v_j)$ with $A(i,j)$. The node degrees $\vd=\{d_1,\cdots,d_N\}$ compute the sum of edge weights connected to node $i$ by $d_i$. The matrix form of $\vd$ is defined as $\mD$ whose diagonal elements are obtained from $\vd$. The symbol ``$\sim$'' denotes 
random sampling.

\textbf{GCN.}
As originally developed by~\cite{Kipf2017}, GCN recursively updates
\begin{eqnarray}
\label{Eq:gcn}
\mH^{(l+1)} &=& \sigma\left(\hat{\mA}\mH^{(l)}\mW^{(l)}\right),
\end{eqnarray}
where $\mH^{(l+1)}=\{\vh_1^{(l+1)},\cdots,\vh_N^{(l+1)}\}$ are the hidden vectors of the $l$-th layer with $\vh_i^{(l)}$ as the hidden feature for node $i$; $\hat{\mA}=\hat{\mD}^{-1/2}(\mA+\mI)\hat{\mD}^{-1/2}$ is the re-normalization of the adjacency matrix, and $\hat{\mD}$ is the degree matrix of $\mA+\mI$; $\sigma(\cdot)$ is a nonlinear function, \emph{i.e.} the ReLU function; and $\mW^{(l)}\in\mathbb{R}^{C_{l} \times C_{l+1}}$ is the filter matrix in the $l$-th layer with $C_l$ being the feature dimension. We assume $C_l$ of all layers to be the same, $C_l=C$ in our following context unless otherwise specified.  

In contrast to MLP, GCN by additionally performing message passing along $\hat{\mA}$ is able to capture the topology of the input graph. Yet, this will unfortunately incur the over-smoothing issue for deep GCNs. We review the conclusion by~\cite{Oono2020} and repeat its compact form by~\cite{rong2020dropedge} here.
\begin{theorem} [Over-Smoothing~\cite{Oono2020}]
\label{col:1}
Let $\lambda_1\le\cdots\le\lambda_N$ be the eigenvalues of $\hat{\mA}$, sorted in ascending order. Suppose the multiplicity of the largest eigenvalue $\lambda_N$ is $M(\le N)$, i.e., $\lambda_{N-M}<\lambda_{N-M+1}=\cdots=\lambda_N$ and the second largest eigenvalue is defined as $\lambda :=  \max_{n=1}^{N-M}|\lambda_n|$, and let $\mE\in\R^{N\times M}$ to be the eigenspace associated with $\lambda_{N-M+1},\cdots,\lambda_{N}$. We denote the maximum singular value of $\mW_l$ by $s_l$ (assuming $s_l\leq1$ here).
Then we have $\lambda<\lambda_{N-M+1}=\cdots=\lambda_N=1$, and
\begin{equation}
    d_\mathcal{M}(\mH^{(l)}) \le s_l\lambda d_\mathcal{M}(\mH^{(l-1)}), \label{equ:distance}
\end{equation}
where $\mathcal{M}\coloneqq\{\mE\mC|\mC\in \R^{M\times C}\}$, and the Frobenius-norm-induced distance $d_\mathcal{M}(\mH):=\inf_{\mY\in \mathcal{M}} ||\mH-\mY||_\mathrm{F}$. 
That is to say, as $l\rightarrow\infty$, the output of GCN on $\gG$ exponentially approaches a lower-dimensional space $\gM$ that is independent to the input $\bm{X}$, incurring over-smoothing.
\end{theorem}

This theorem states that the spectral of $\hat{\mA}$ has a crucial impact on over-smoothing. A potential way to alleviate over-smoothing (regardless of $s_l$ that is affected by training) is by enlarging the value of $\lambda$ to alleviate the condition $\lambda <1$, which is the goal of DropEdge~\cite{rong2020dropedge}.

\textbf{DropEdge.}
\cite{rong2020dropedge} found that the number of edges is negatively related to the value of $\lambda$, and dropping edges to a certain extent will possibly increase the value of $\lambda$ thus relieving over-smoothing. In practice, they try a random dropping process to prevent over-fitting as well. At each training iteration, DropEdge uniformly drops a subset of edges of a preset size $p|\gE|$ by a non-replacement sampling process, \emph{i.e.},
\begin{eqnarray}
\label{Eq:DropEdge}
\mA_{\text{drop}}\sim\text{DropEdge}(\mA, p|\gE|),
\end{eqnarray}
which can be also represented in a probability form, $\mA_{\text{drop}}(i,j) = \mA(i,j)*\text{Bernoulli}(1-p)$. In practice, 
The re-normalization trick is further conducted on $\mA_{\text{drop}}$, leading to $\hat{\mA}_{\text{drop}}$. The authors also prove in theory the benefits of using DropEdge (see~\cite{rong2020dropedge,huang2020tackling} for more details).

\section{Proposed Method: DropEdge++}
This section presents the details of DropEdge++, including the two enhanced samplers:  Layer-Dependent (LD) sampler and Feature-Dependent (FD) sampler. We also provide essential theoretical justifications.

\subsection{Layer-Dependent Sampler}
\label{sec:ld}
DropEdge only sparsifies the adjacency matrix once and keeps it unchanged for all layers. One might be interested in if we can remove this constraint and adopt different adjacency matrices for different layers to involve more flexibility. Inspired by this, Layer-Independent (LI) DropEdge proposed by~\cite{rong2020dropedge} constructs the different-layer adjacency matrix independently. However, LI is still unable to study how different-layer adjacency influences each other and what relation in-between helps prevent over-smoothing. This sub-section investigates these questions in the form of layer-dependent DropEdge. 
Prior to going further, we first reformulate GCN (Eq.~(\ref{Eq:gcn})) as the following Layer-wise GCN (L-GCN):
\begin{eqnarray}
\label{Eq:l-gcn}
\mH^{(l+1)} &=& \sigma\left(\mA^{(l)}\mH^{(l)}\mW^{(l)}\right),
\end{eqnarray}
where the adjacency matrix $\mA^{(l)}$ varies for different layer. In LI~\cite{rong2020dropedge}, $\mA^{(l)}$ is given by performing Eq.~(\ref{Eq:DropEdge}) for each layer independently. Different from LI, we introduce two variants of layer-dependent methods below: Layer-Increasingly-Dependent (LID) DropEdge and Layer-Decreasingly-Dependent (LDD) DropEdge.

\textbf{LID} creates the adjacency matrix of the current layer by dropping a fixed number of edges conditional on the higher layer. In form, the adjacency matrix of the $l$-th layer is recursively computed by
\begin{eqnarray}
\label{eq:lid}
\mA_{\text{drop}}^{(l)} \sim \text{DropEdge}(\mA_{\text{drop}}^{(l+1)}, \delta p|\gE|),
\end{eqnarray}
where $\delta p=(p_{\max}-p_{\min})/(L-1)$, and $p_{\min}$ and $p_{\max}$ are hyper-parameters. For the top layer, $\mA_{\text{drop}}^{(L-1)}\sim \text{DropEdge}(\mA, p_{\text{min}}|\gE|)$. Obviously, we have $\mA_{\text{drop}}^{(0)}\subset\mA_{\text{drop}}^{(1)}\subset\cdots\subset\mA_{\text{drop}}^{(L-1)}$, from sparse to dense. 


\textbf{LDD}, in contrast to LID, derives the adjacency matrix of the current layer upon the lower layer, 
\begin{eqnarray}
\label{eq:ldd}
\mA_{\text{drop}}^{(l)} \sim \text{DropEdge}(\mA_{\text{drop}}^{(l-1)}, \delta p|\gE|),
\end{eqnarray}
and $\mA_{\text{drop}}^{(0)}\sim \text{DropEdge}(\mA, p_{\text{min}}|\gE|)$.
Clearly, $\mA_{\text{drop}}^{(0)}\supset\mA_{\text{drop}}^{(1)}\supset\cdots\supset\mA_{\text{drop}}^{(L-1)}$, from dense to sparse. Both LID and LDD perform adjacency normalization and compute $\mA^{(l)}=\hat{\mA}_{\text{drop}}^{(l)}$ for Eq.~(\ref{Eq:l-gcn}). The mechanisms of LI, LID, and LDD are depicted in Fig.~\ref{fig.layerwise_samplers}.

Now, we characterize the theoretical property of LID and LDD. Note that it is no longer feasible to apply Theorem~\ref{col:1} for analyzing the over-smoothing of LID and LDD because $\mA_{\text{drop}}^{(l)}$ is varying over layers and the subspace which GCN will converge to is unknown. To overcome this difficulty, we first revisit why over-smoothing hinders the training. We compute the gradient of the training loss $\gL$ \emph{w.r.t.} $\mW^{(l)}$ in Eq.~(\ref{Eq:l-gcn}):
\begin{eqnarray}
\label{Eq:gradient}
\frac{\partial\gL}{\partial\mW^{(l)}} = (\mZ^{(l)})^{\mathrm{T}}\frac{\partial\gL}{\partial\mU^{(l)}},
\end{eqnarray}
where $\gL$ is the training loss, $\mZ^{(l)}=\mA^{(l)}\mH^{(l)}$ and $\mU^{(l)}=\mA^{(l)}\mH^{(l)}\mW^{(l)}$. If over-smoothing happens and $\mZ^{(l)}$ locates in a low-dimension space $\gM$, performing the inner-product $(\mZ^{(l)})^{\mathrm{T}}\frac{\partial\gL}{\partial\mU^{(l)}}$ dramatically decreases the magnitude of $\frac{\partial\gL}{\partial\mW^{(l)}}$, as projecting $\frac{\partial\gL}{\partial\mU^{(l)}}$ into $\gM$ will lose its values of most dimensions. It is thus difficult to train the model. 

Next, we discuss the behavior of $\mZ^{(l)}$ when $l\rightarrow\infty$. We approximate the analysis by relaxing the activation function to be linear, then the accumulated adjacency matrix from the input layer to the current layer becomes $\prod_{i=0}^{l}\mA^{(i)}$. Based on the derivation in~\cite{rong2020dropedge}, the edge number of the adjacency matrix is positively related to the speed of over-smoothing as shown in Theorem~\ref{col:1}. This motivates us to leverage the edge number of $\prod_{i=0}^{l}\mA^{(i)}$ to measure the degree of over-smoothing for current layer $\mZ^{(l)}$. Considering the updates of Eq.~\ref{Eq:gradient} for all layers, we compute the Mean Edge Number (MEN) of the accumulated adjacency matrices and introduce a new concept as follow.

\begin{definition} [MEN]
\label{de:men}
Suppose $|\mA|$ returns the edge number of $\mA$. Then, the MEN of L-GCN is defined as $\text{MEN}=\frac{1}{L}\sum_{l=0}^{L-1}|\prod_{i=0}^{l}\mA^{(i)}|$. 
\end{definition}

Despite its approximation of linear activation in Eq.~(\ref{Eq:l-gcn}), MEN is able to provide us with a portable tool to theoretically investigate the convergence behavior of L-GCN given different edge samplers. More importantly, our experiments does verify that MEN is positively related to over-smoothing, and lower MEN yields better performance (see Fig.~\ref{fig.MEN_simple}).

\begin{figure}[t]
\centering
\includegraphics[width=0.40\textwidth]{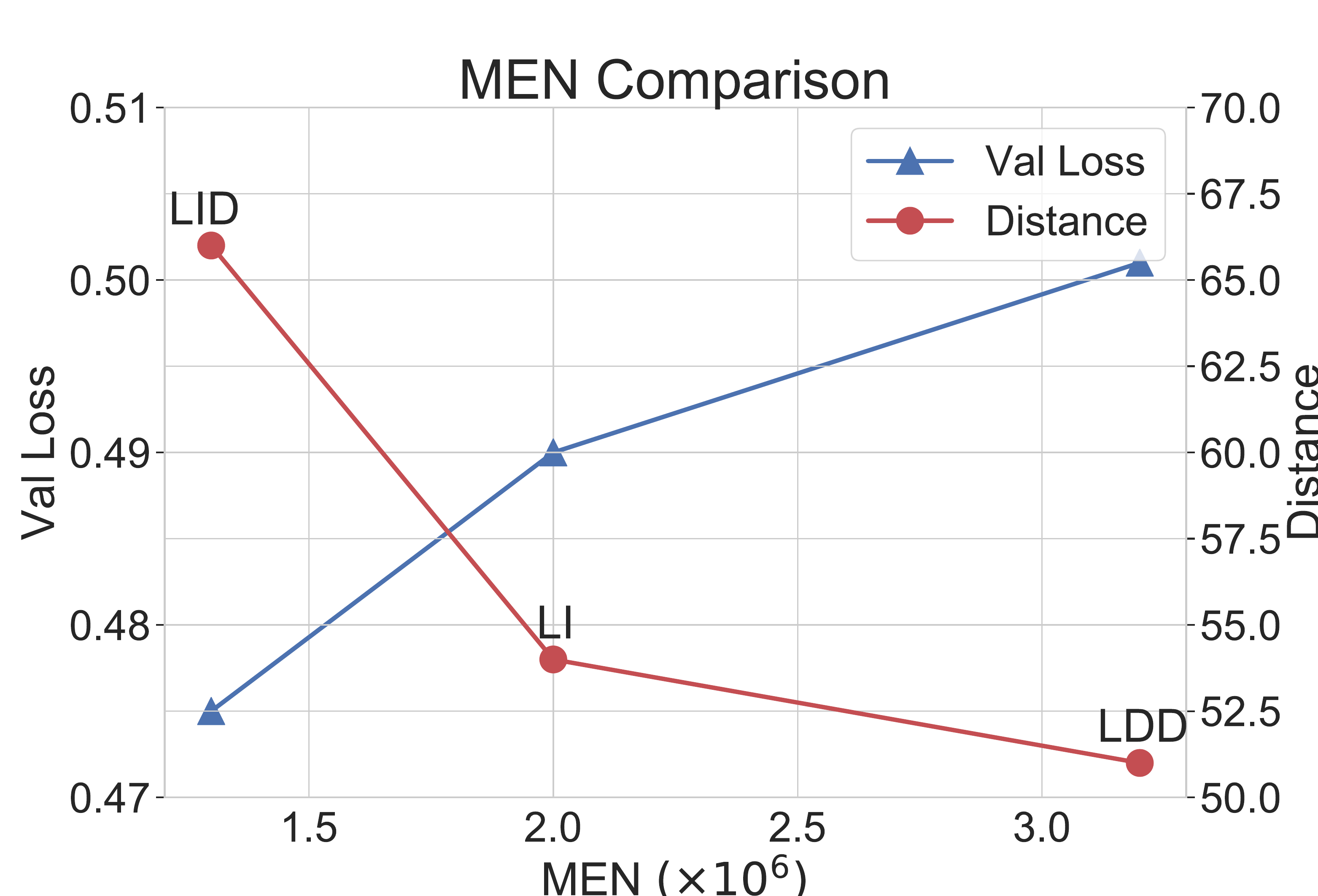}
\caption{MEN analysis upon 16-layer GCN on Cora. Distance is given by $\frac{1}{L-3} \sum \limits_{l=1} ^ {L-3} D^{(l)}$, where $D^{(l)} = ( \sum \limits_{i,j} ( \mH^{(l+1)}_{ij} - \mH^{(l)}_{ij} ) ^2 ) ^ {\frac{1}{2}}$ returns the distance between the activations of adjacent layers. Obviously, less MEN derives larger Distance, implying less over-smoothing, hence smaller loss and better performance. More cases are shown in Fig.~\ref{fig.MEN_detail_analysis}. }
\label{fig.MEN_simple}
\end{figure}


With the notion of MEN, we immediately have,
\begin{theorem}
\label{th:mean}
Let $\text{MEN}_{\text{lid}}$ and $\text{MEN}_{\text{ldd}}$ denote the expected MENs of the adjacency matrices sampled by LID and LDD, respectively, under the same values of $p_{\text{min}}$ and $p_{\text{max}}$. Then, $\text{MEN}_{\text{lid}}\leq \text{MEN}_{\text{ldd}}$.
\end{theorem}

\begin{proof}
    Without loss of generality, we assume the adjacency matrices sampled by LID and LDD are in the reverse order to each other, namely, 
\begin{eqnarray}
\nonumber
\mathbb{A}_{\text{lid}} &=& \{\mA_0, \mA_1, ..., \mA_{L-1}\}, \\    
\nonumber
\mathbb{A}_{\text{ldd}} &=& \{\mA_{L-1}, \mA_{L-2}, ..., \mA_{0}\},
\end{eqnarray}
where 
$\mA_{0} \subset \mA_{1} \subset ... \subset \mA_{L-1}$. Further, we denote the accumulated products as $\mA'_l = \prod \limits_{i=0}^l \mA_i$ and $\mA''_l =\prod \limits_{i=L-1}^{L-l-1} \mA_i$. We arrive at $\text{MEN}_{\text{lid}} = \frac{1}{L}\sum_{l=0}^{L-1} |\mA'_l|,\text{MEN}_{\text{ldd}}=\frac{1}{L}\sum_{l=0}^{L-1} |\mA''_l|$.

In order to prove $\text{MEN}_{\text{lid}}\leq \text{MEN}_{\text{ldd}}$, we prove $|\mA'_l|\leq|\mA''_l|$ for each $l\in[0,L-1]$. 

Note that $\mA_{i-L+l+1} \subset \mA_{i}$ since  $i-L+l+1\leq i$, for $0\leq l\leq L-1$. So, we can define $\mA_{i}=\mA_{i-L+l+1}+\delta\mA_{i-L+l+1}$ for abbreviation, where $\delta\mA_{i-L+l+1}\geq 0$. Putting this into $\mA''_{l}$, we easily devise
\begin{eqnarray}
\nonumber
\mA''_l &=& \prod \limits_{i=L-1}^{L-l-1} \mA_i, \\
\nonumber
&=& \prod \limits_{i=L-1}^{L-l-1} (\mA_{i-L+l+1}+\delta\mA_{i-L+l+1}),\\
\nonumber
&=& \mB + \prod \limits_{i=L-1}^{L-l-1}\mA_{i-L+l+1} ,\\
\nonumber
&=& \mB + \prod \limits_{i=l}^0 \mA_i, \\ 
\nonumber
&=& \mB + (\mA'_l)^{\mathrm{T}},
\end{eqnarray}
where $\mB\geq0$, and the last equivalence holds due to the symmetry of each adjacency matrix.
Hence, it is natural to see that $|\mA'_l|\leq|\mA''_l|$, which concludes the proof.
\end{proof}

It suggests that LID incurs the least MEN and delivers the best trainability and generalization than LDD. Note that Theorem~\ref{th:mean} does not mean smaller MEN always leads to better performance. For instance, dropping all edges produces zero MEN and completely eliminates over-smoothing, but it yields performance detriment as well, since the model has degenerated to an MLP that is not topology-aware. Actually, we should choose edge sampling rate for each particular case to balance between reducing over-smoothing and topology modeling. In our experiments, the best sampling rate is determined by validation.

LI is a layer-wise version of DropEdge, it is also valuable to justify how LID performs compared to LI. Yet it is not trivial to directly contrast their MENs in theory. Instead, we draw their correlation in terms of an upper-bound of MEN, defined as $\text{MEN}^{\ast}=\frac{1}{L}\sum_{l=0}^{L-1}\prod_{i=0}^{l}|\mA^{(i)}|$.
\begin{theorem}
\label{th:mean-li}
Suppose LI shares the same dropping rate as LID for each layer, and let $\text{MEN}^{\ast}_{\text{lid}}$ and $\text{MEN}^{\ast}_{\text{li}}$ denote the expected $\text{MEN}^{\ast}$ of LID and LI, respectively. Then, $\text{MEN}^{\ast}_{\text{lid}}\leq \text{MEN}^{\ast}_{\text{li}}$. 
\end{theorem}

\begin{proof}
Denote $P(m_{l-1}|m_{l})$ as the probability of the edge number of the $(l-1)$-th layer being $m_{l-1}$, conditioned on the edge number of the $l$-th layer being $m_{l}$. Further denote $|\mA| = n$, and the rate of preserved edges for the $l$-th layer as $p_l$.
\begin{eqnarray}
\nonumber
&& \mathbb{E}_{\text{lid}} \left( \prod \limits_{i=0}^l |\mA_{\text{drop}}^{(i)}| \right) \\
\nonumber
&=&\sum \limits_{
\begin{tiny}
{m_0\leq m_1 \leq ... \leq m_l}\end{tiny}
}\left(P(m_l)\prod \limits_{i=0} ^ {l-1} P(m_{l-i-1}|m_{l-i})  \prod \limits_{i=0}^l m_i\right), \\
\nonumber
&=& \sum \limits_{
\begin{tiny}
{m_1 \leq ... \leq m_l}\end{tiny}
}\left(P(m_l)\prod \limits_{i=0} ^ {l-2} P(m_{l-i-1}|m_{l-i})\prod \limits_{i=1}^l m_i\right) \\
\nonumber 
&& \cdot \sum \limits_{\begin{tiny}{m_0\leq m_1}\end{tiny}} P(m_0|m_1)m_0,\\
\nonumber
&=& \sum \limits_{
\begin{tiny}
{m_1 \leq ... \leq m_l}\end{tiny}
}\left(P(m_l)\prod \limits_{i=0} ^ {l-2} P(m_{l-i-1}|m_{l-i})\prod \limits_{i=1}^l m_i\right) p_0m_1,\\
\nonumber
&\leq& \sum \limits_{
\begin{tiny}
{m_1 \leq ... \leq m_l}\end{tiny}
}\left(P(m_l)\prod \limits_{i=0} ^ {l-2} P(m_{l-i-1}|m_{l-i}) \prod \limits_{i=1}^l m_i\right)p_0 n,\\
\nonumber
&\leq& \cdots \\
\nonumber
&\leq& \prod \limits_{i=0}^l \left( p_i n \right), \\
\nonumber
&=& \mathbb{E}_{\text{li}} \left( \prod \limits_{i=0}^l |\mA_{\text{drop}}^{(i)}| \right). 
\end{eqnarray}

Hence, we have $\text{MEN}^{\ast}_{\text{lid}}\leq \text{MEN}^{\ast}_{\text{li}}$.

\end{proof}

\subsection{Feature-Dependent Sampler}
\label{sec:FD}
In this section, we introduce another sampler, FD. Unlike DropEdge in Eq.~(\ref{Eq:DropEdge}) where each edge is sampled uniformly, FD draws edges according to a measurement by a kernel function of node feature pairs. We define the kernel function as $K(\bm{x}_1,\bm{x}_2)$ that returns the similarity of two node features $\bm{x}_1$ and $\bm{x}_2$, for instance, $K(\bm{x}_1,\bm{x}_2)=\bm{x}_1^{\mathrm{T}}\bm{x}_2$ 
in the linear case. Then, $\mA_{\text{drop}}$ is constructed as follow,
\begin{eqnarray}
\label{eq:FD}
\mA_{\text{drop}}\sim\text{FD}(\mA,\mX, p|\gE|),
\end{eqnarray}
in other words,
\begin{eqnarray}
\mA_{\text{drop}}(i,j) =  \mA(i,j)*\text{Bernoulli}(\frac{(1-p)K(\bm{x}_i,\bm{x}_j)}{\sum_{(i,j)\in\gE}{K(\bm{x}_i,\bm{x}_j)}})
\end{eqnarray}
Similar to DropEdge, we employ the non-replacement sampling and utilize the re-normalization trick on $\mA_{\text{drop}}$ to obtain $\hat{\mA}_{\text{drop}}$. The kernel computation is economical with the complexity of $O(|\mathcal{E}|)$. We pre-compute the kernels to eliminate the need of computing them at each training iteration in our implementation. 
The kernel here is applied to yield the sampling distribution for the edges, different from KerGNNs~\cite{feng2022kergnns} which leverage graph kernels to compute the convolution between the local subgraph and the trainable graph filters.

Intuitively, the edges connecting two nodes of similar features are more potentially intra-class connections and should be preserved more often, while the edges connecting dissimilar nodes should be possibly removed as they are likely linking nodes of distinct classes. Our FD sampler does serve this goal by employing the kernel function as an additional weight of the edge sampling probability. 
In our experiments (\textsection~\ref{sec:FD_compare}), FD outperforms all other counterparts, which verifies the validity of our design.  

Furthermore, FD correlates the adjacency matrix with the node features. This is beneficial since the output of GCN still involves the information of the input features even if the over-smoothing is unavoidable, according to Theorem~\ref{col:1}. We summarize this property as a theorem for better readability. 
\begin{theorem} 
\label{th:FD}
FD correlates the subspace $\gM$ with the input features $\mX$. To be specific, the bases of $\gM$ are returned as $\ve_m=\mD^{\frac{1}{2}}\vu_m$ for $m=1,\cdots, M$, where the i-th expected diagonal element of $\mD$ is $\mathbb{E}[D_i]=(1-p)\sum_{j}\frac{K(\bm{x}_i,\bm{x}_j)\mA(i,j)}{\sum_{(i,j)\in\gE}{K(\bm{x}_i,\bm{x}_j)}}$ , and $\vu_m$ is the indicator vector of the $m$-th connected component, \emph{i.e.} $u_{m}(i)=1$ if node $i$ belongs to $m$ and $u_{m}(i)=0$ otherwise.
\end{theorem}

\begin{proof}

We borrow Proposition 1~\cite{Oono2020} and find that the convergence subspace $\gM$ of GCN is formulated by the bases of  $\ve_m=\mD^{\frac{1}{2}}\vu_m$ for $m=1,\cdots, M$ where $\mD$ is the degree matrix of the normalized adjacency matrix of the input graph $\gG$, and $\vu_m$ is the indicator vector of the $m$-th connected component of $\gG$, \emph{i.e.} $u_{m}(i)=1$ if node $i$ belongs to $m$ and $u_{m}(i)=0$ otherwise. As it also holds for the graph $\gG$ after performing FD, the bases of convergence subspace $\gM$ become $\ve_m=\mD^{\frac{1}{2}}\vu_m$ for $m=1,\cdots, M$. It is natural to see that $\mathbb{E}[D_i]=\mathbb{E}[\sum_{j}
   \mA_{\text{drop}}(i,j)]=(1-p)\sum_{j}\frac{K(\bm{x}_i,\bm{x}_j)\mA(i,j)}{\sum_{(i,j)\in\mathcal{E}}{K(\bm{x}_i,\bm{x}_j)}}$.
\end{proof}
By Theorem~\ref{th:FD}, our FD sampler can bridge the dependence between the convergence subspace and the input features, as opposed to DropEdge. This better relieves the effect of over-smoothing since the output after an infinite number of layers is still related to the input features.

\begin{algorithm}[t] 
\caption{The sampling method of DropEdge++}
 {\bf Input:}
The adjacency matrix $\mA$;
The feature matrix $\mX$;\\
The set of edges $\gE$;
The number of layers $L$;\\
The sampling rates $p_{\text{min}}, p_{\text{max}}$;\\
 {\bf Output:}
The series of the sampled adjacency matrices $\mathbb{A} = \{\mA_{\text{drop}}^{(0)},\cdots,\mA_\text{drop}^{(L-1)} \}$;
\begin{algorithmic}[1] 

\State Initialize $\mathbb{A} = \{ \}$
\State Compute $\delta p = (p_\text{max} - p_\text{min}) / (L - 1)$
\State Sample $\mA_\text{drop}^{(L-1)}\sim\text{FD}(\mA,\mX, p_\text{min}|\gE|)$ 
\State $\mathbb{A} = \{ \mA_\text{drop}^{(L-1)} \}$
\For {$l=L-2$ to $0$}
\State Sample $\mA_{\text{drop}}^{(l)} \sim \text{FD}(\mA_{\text{drop}}^{(l+1)}, \mX,\delta p|\gE|)$
\State $\mathbb{A} = \mathbb{A} \cup \{ \mA_\text{drop}^{(l)} \}$
\EndFor \\
\Return
$\mathbb{A}$

\end{algorithmic} 
\label{alg:ours}
\end{algorithm}

\textbf{Overall Framework.}
Our final version of DropEdge++, given the analyses above, is the integration of LID and FD. Specifically, we employ Eq.~(\ref{eq:lid}) but using
$\mA_{\text{drop}}^{(l)} \sim \text{FD}(\mA_{\text{drop}}^{(l+1)}, \mX,\delta p|\gE|)$. The detailed algorithmic flow is depicted in Algorithm~\ref{alg:ours}. Note that Theorem~\ref{th:mean} still holds under the FD sampler.    

\textbf{Complexity.} Before training, DropEdge++ requires a pre-computation of the feature similarity for the EB sampler, which is of complexity $O(|\gE|)$. During training, in each epoch, the LID sampler processes the sampling procedure in a layer-wise manner, resulting in a complexity of $O(L|\gE|)$. Since the sampling does not depend on the parameters in the GCNs and is not involved in the computation graph, the overhead brought by DropEdge++ is negligible compared with training the backbones, which is further verified in experiments in~\textsection~\ref{sec:time}.

\section{Experiments}

\subsection{Datasets}
To demonstrate the generality of our proposed method, we evaluate DropEdge++ on six node classification benchmarks, including both full- and semi- supervised tasks. 

\begin{table*}[htbp]
   \centering
   \caption{Datasets Statistics. ``Full'' and ``Semi'' denote full-supervised and semi-supervised style, respectively.}
     \begin{tabular}{l|rrrrc}
     \hline
     Dataset & Node  & Edge  & Feature & Class & Training / Validation / Testing \bigstrut[t]\\
     \hline
     Cora (Full)  & 2,708 & 5,429 & 1,433 & 7     & 1,208 / 500 / 1,000 \bigstrut[t]\\
     Citeseer (Full) & 3,327 & 4,732 & 3,703 & 6     & 1,812 / 500 / 1,000 \bigstrut[t]\\
     Pubmed (Full) & 19,717 & 44,338 & 500   & 3     & 18,217 / 500 / 1,000\bigstrut[t]\\
     \hline
     Cora (Semi)  & 2,708 & 5,429 & 1,433 & 7     & 140 / 500 / 1,000 \bigstrut[t]\\
     Citeseer (Semi) & 3,327 & 4,732 & 3,703 & 6     & 120 / 500 / 1,000 \bigstrut[t]\\
     Pubmed (Semi) & 19,717 & 44,338 & 500   & 3     & 60 / 500 / 1,000\bigstrut[t]\\
     \hline
     Coauthor CS & 18,333 & 81,894 & 6,805 & 15    & 300 / 450 / 17,583 \bigstrut[t]\\
     Coauthor Physics & 34,493 & 247,962 & 8,415 & 5     & 100 / 150 / 34,243 \bigstrut[t]\\
     Amazon Photos & 7,487 & 119,043 & 745   & 8     &  160 / 240 / 7,087 \bigstrut[t]\\
     \hline
     \end{tabular}%
   \label{tab:datasets}%
 \end{table*}

\textbf{Citation Networks}
The Cora, Citeseer, and Pubmed~\cite{sen2008collective} datasets are popular benchmarks for graph neural networks. The nodes represent papers while features are the bag-of-words representation of paper titles. Edges stand for the citation relationship between the papers, and the papers are labeled by the categories they belong to. On these three datasets, we leverage both the
semi-supervised fashion according to the standard split by~\cite{yang2016revisiting, DBLP:journals/corr/abs-1710-10903, Kipf2017, chen2020simple} and the full-supervised mode following~\cite{chen2018fastgcn,Huang2018,rong2020dropedge}.

\textbf{Co-author and Co-purchase Networks}
The Co-author CS and Co-author Physics are co-author networks~\cite{shchur2018pitfalls} based on the Microsoft Academic Graph with nodes representing authors and edges representing the co-authorship on a paper. Features are constructed by the keywords for the authors' papers, and labels reflect the authors' active fields. The Amazon Photos~\cite{shchur2018pitfalls} dataset is derived from Amazon co-purchase graph. Nodes represent goods, and edges represent close co-purchase relationship. Features are the product reviews in bag-of-words form, and labels indicate the category. 
Among these three datasets, we apply the semi-supervised tasks.
We randomly fix 20 nodes per class for training, 30 nodes per class for validation, and all the rest for testing, following~\cite{chen2019measuring,chang2020spectral}.

\subsection{Backbones}
We implement DropEdge++ on 4 popular backbones: generic GCN~\cite{Kipf2017}, ResGCN~\cite{he2016deep}, JKNet~\cite{xu2018representation}, and APPNP* that adds non-linearity to the propagation of APPNP~\cite{klicpera2018predict}.  We list these four propagation models in Appendix. For ResGCN~\cite{he2016deep}, the residual connections within layers are implemented as the adding operation. For JKNet~\cite{xu2018representation}, the final output embedding is derived from the concatenation of the hidden layer outputs. APPNP* is an extension of APPNP~\cite{klicpera2018predict} by adding non-linearity on the propagation model. A similar model to APPNP* has been discussed by~\cite{chen2020simple}, where the authors found that  adding both non-linearity and identity mapping to APPNP results in more desired performance for deep architectures. However, the identity mapping requires careful hyper-parameter tuning and we thus only apply non-linearity for simplicity and find it works well in our experiments. 


\textbf{Self-feature Modeling}
We also implement self-feature modeling~\cite{fout2017pro}
, which is a variant on the basic GCN propagation.
\begin{align}
\label{Eq:self}
    \mH^{(l+1)} = \sigma\left(\hat{\mA}\mH^{(l)}\mW^{(l)} + \mH^{(l)}\mW_{\text{self}}^{(l)} \right)
\end{align}

We can simply replace the convolution operation in GCN with Eq.~(\ref{Eq:self}) to obtain the self-feature modeling variant, which indeed could be well treated as a hyper-parameter for the potential enhancement on GCN, ResGCN, and JKNet.

\begin{figure*}[h!]
\centering
\includegraphics[width=0.32\textwidth]{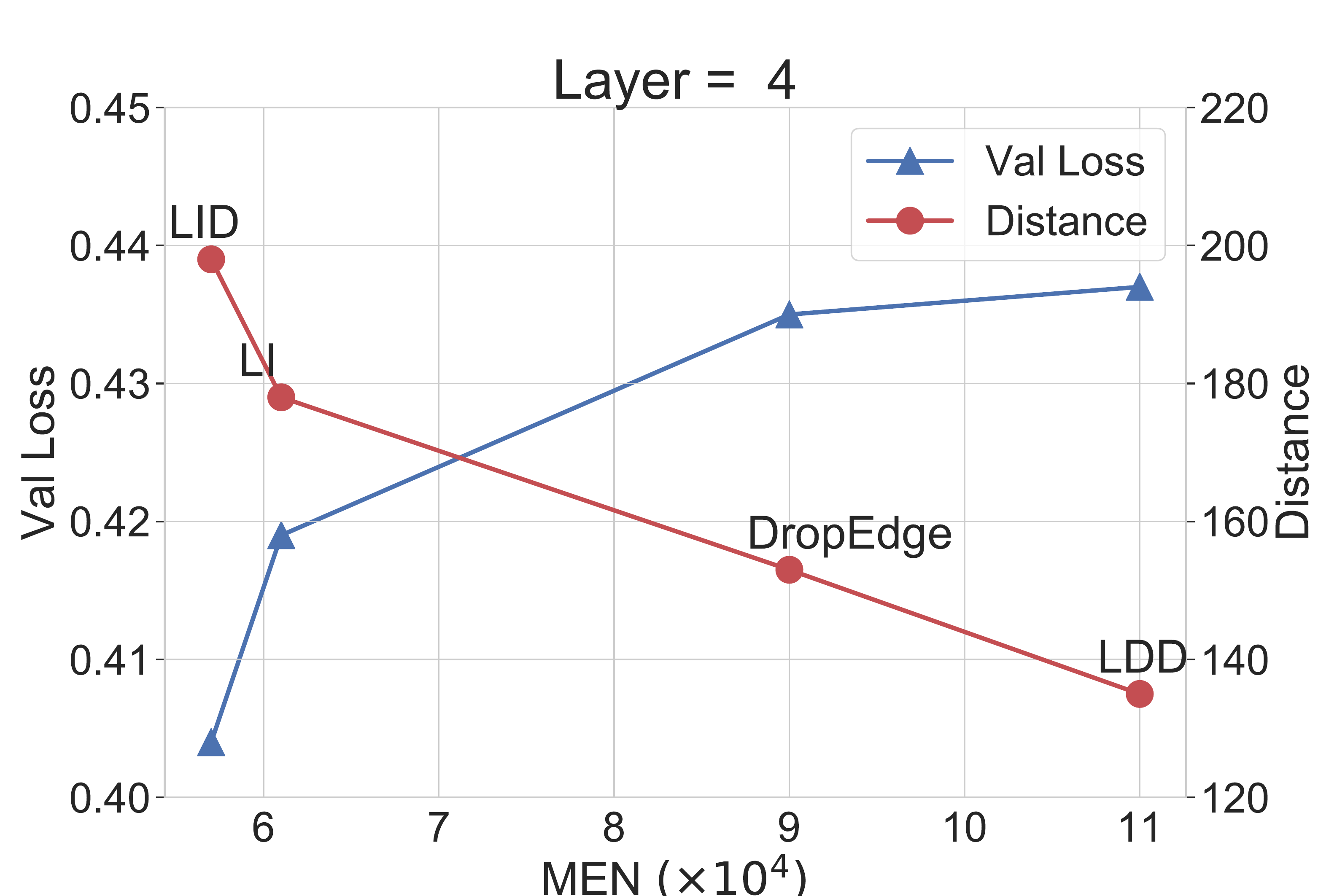}
\includegraphics[width=0.32\textwidth]{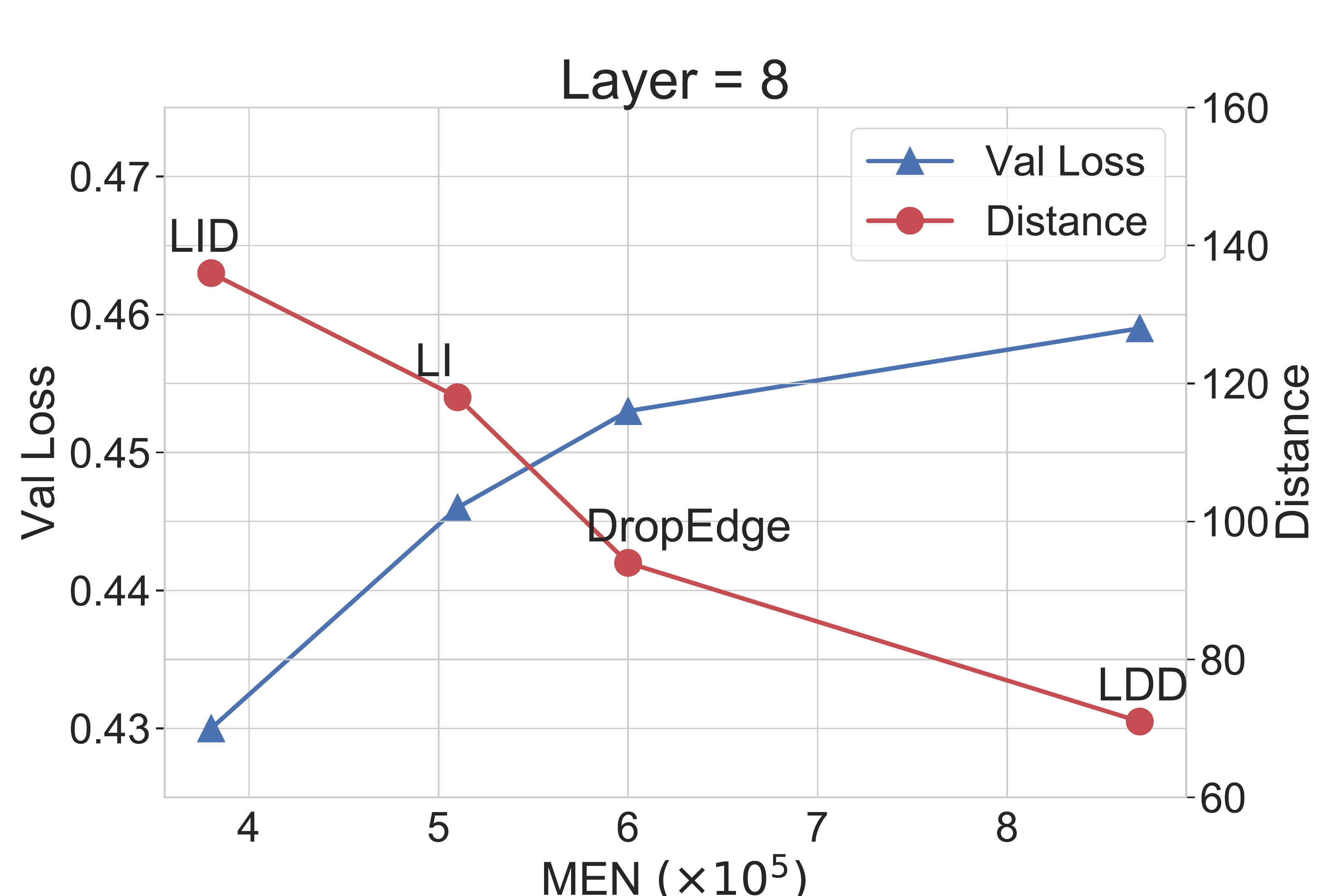}
\includegraphics[width=0.32\textwidth]{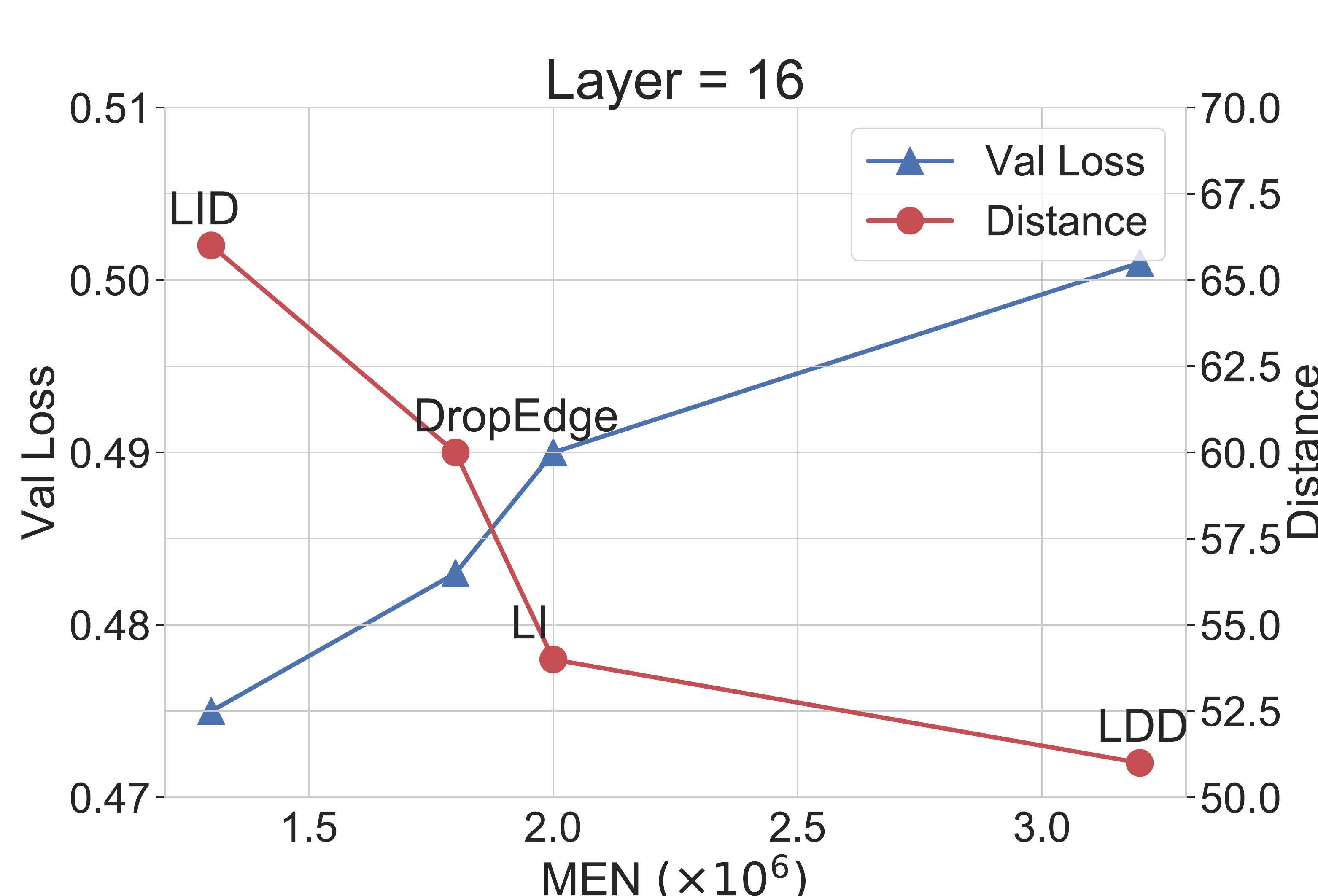}

\caption{MEN vs Validation Loss and Average Distance on full-supervised Cora, a more complete version of Fig.~\ref{fig.MEN_simple}.}
\label{fig.MEN_detail_analysis}
\end{figure*}

\begin{figure*}[t]
\centering
\includegraphics[width=0.32\textwidth]{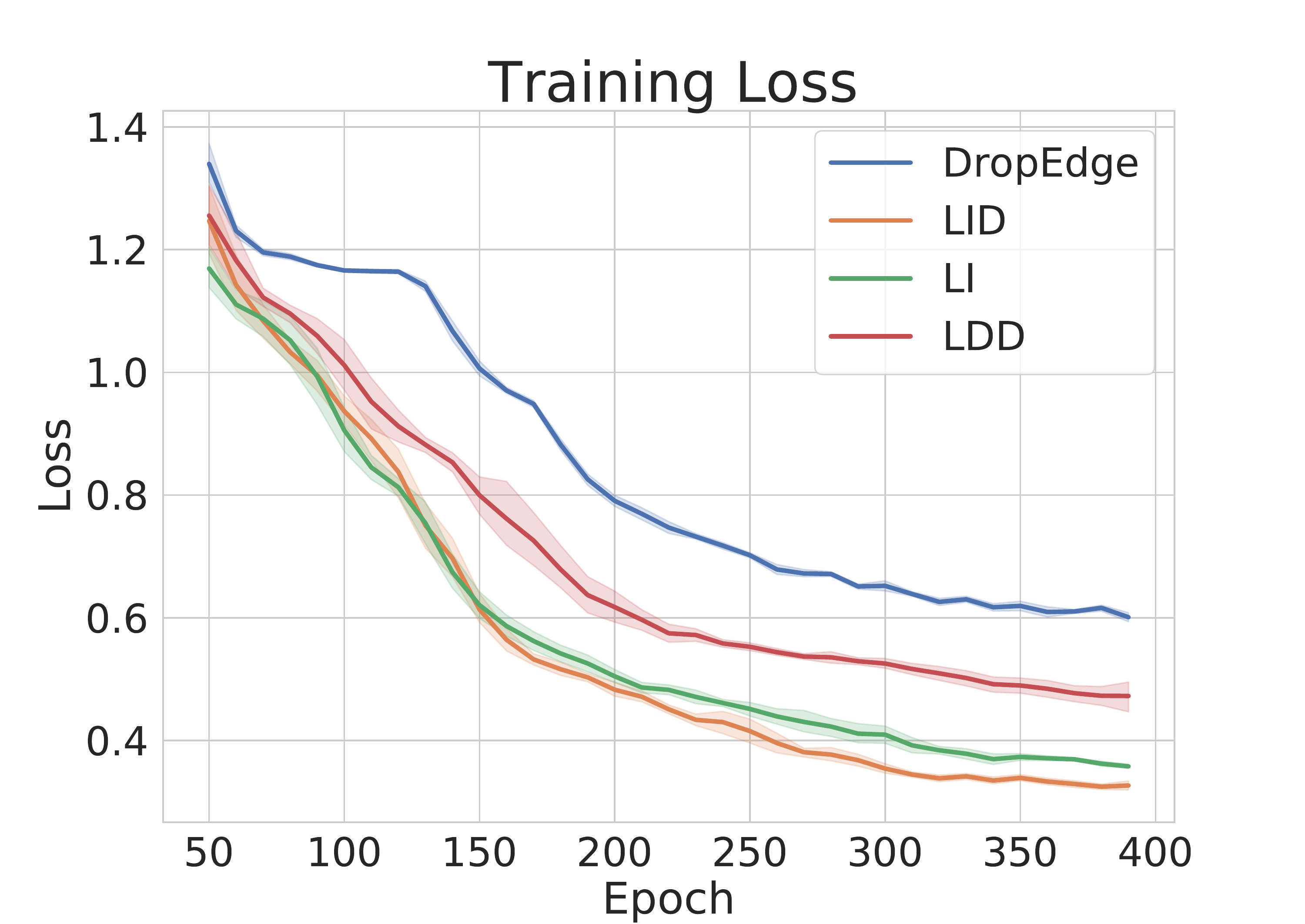}
\includegraphics[width=0.32\textwidth]{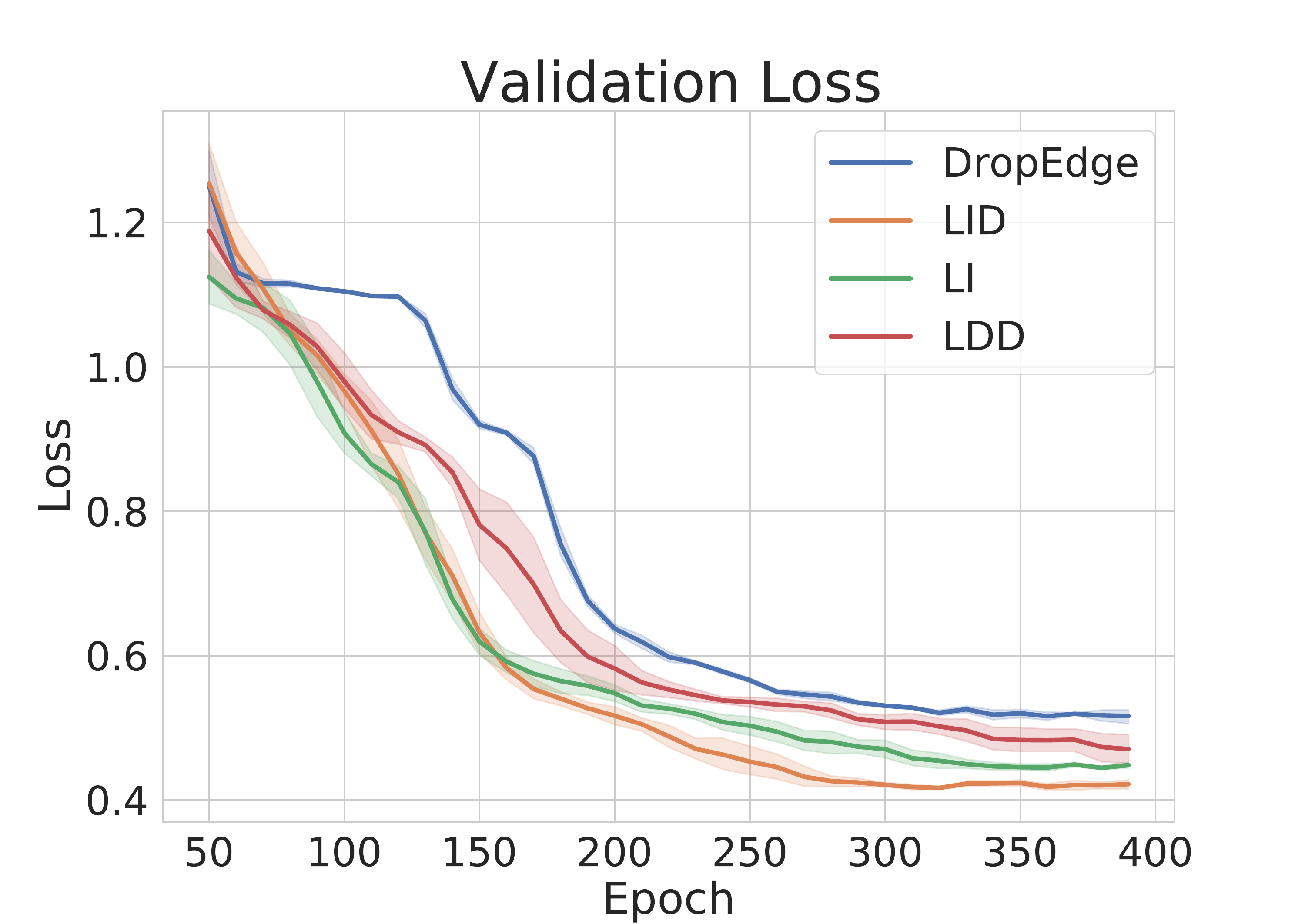}
\includegraphics[width=0.34\textwidth, height=0.22\textwidth]{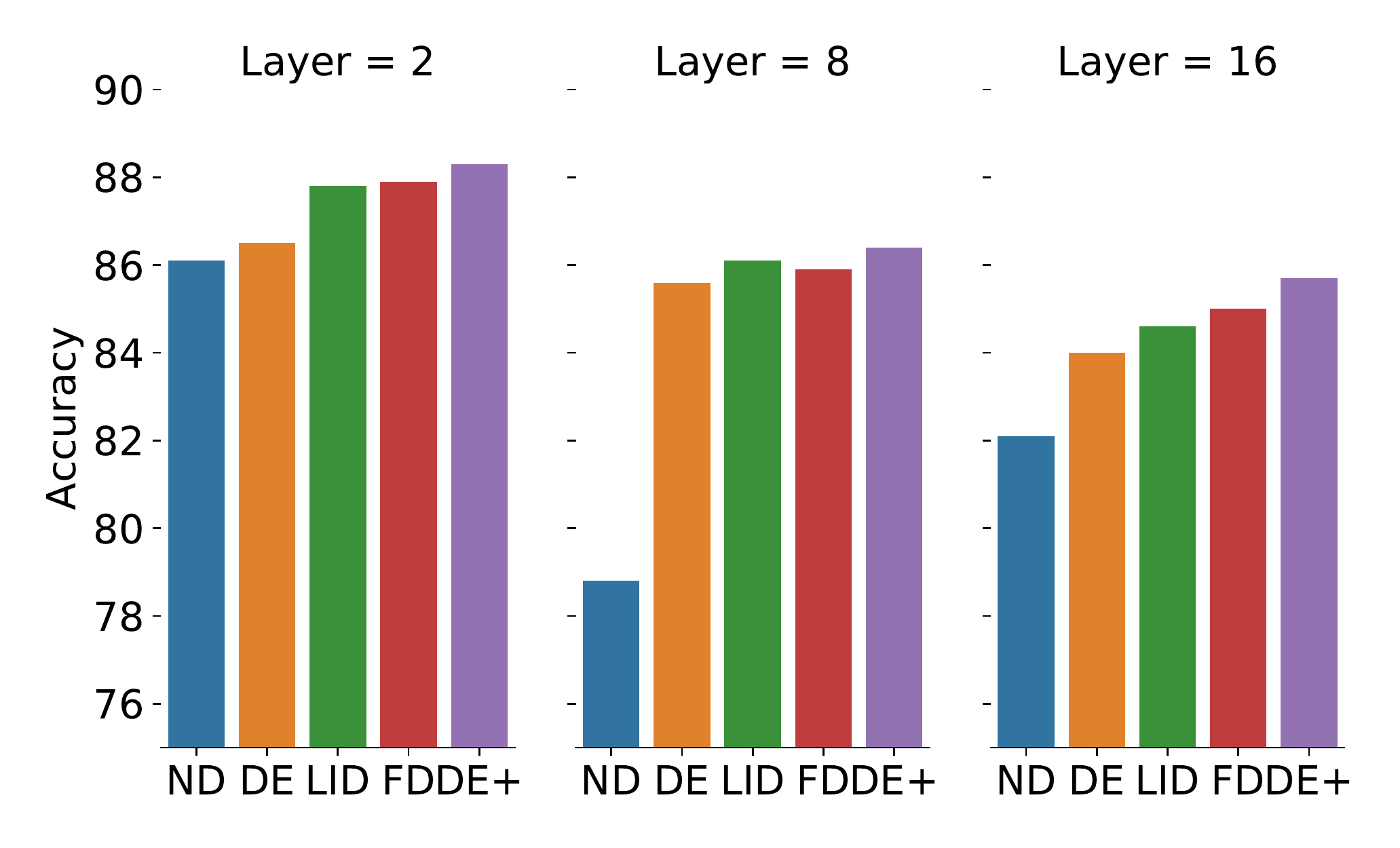}
\caption{ Analysis of LID on full-supervised Cora. 
(1) Left: training loss; Middle: validation loss. LID, LDD and LI share the same sampling rate $p_{\text{max}}$ and $p_{\text{min}}$, while the rate of DropEdge is set to be $(p_{\text{max}}+p_{\text{min}})/2$ for fair comparison. Colored area indicates one standard deviation over 10 runs. (2) Right: Accuracy comparison of different variants (ND: NoDrop, DE: DropEdge, DE+: DropEdge++).
}
\label{fig.compare_layer}
\end{figure*}
\begin{figure*}[h!]
\centering
\includegraphics[width=0.32\textwidth]{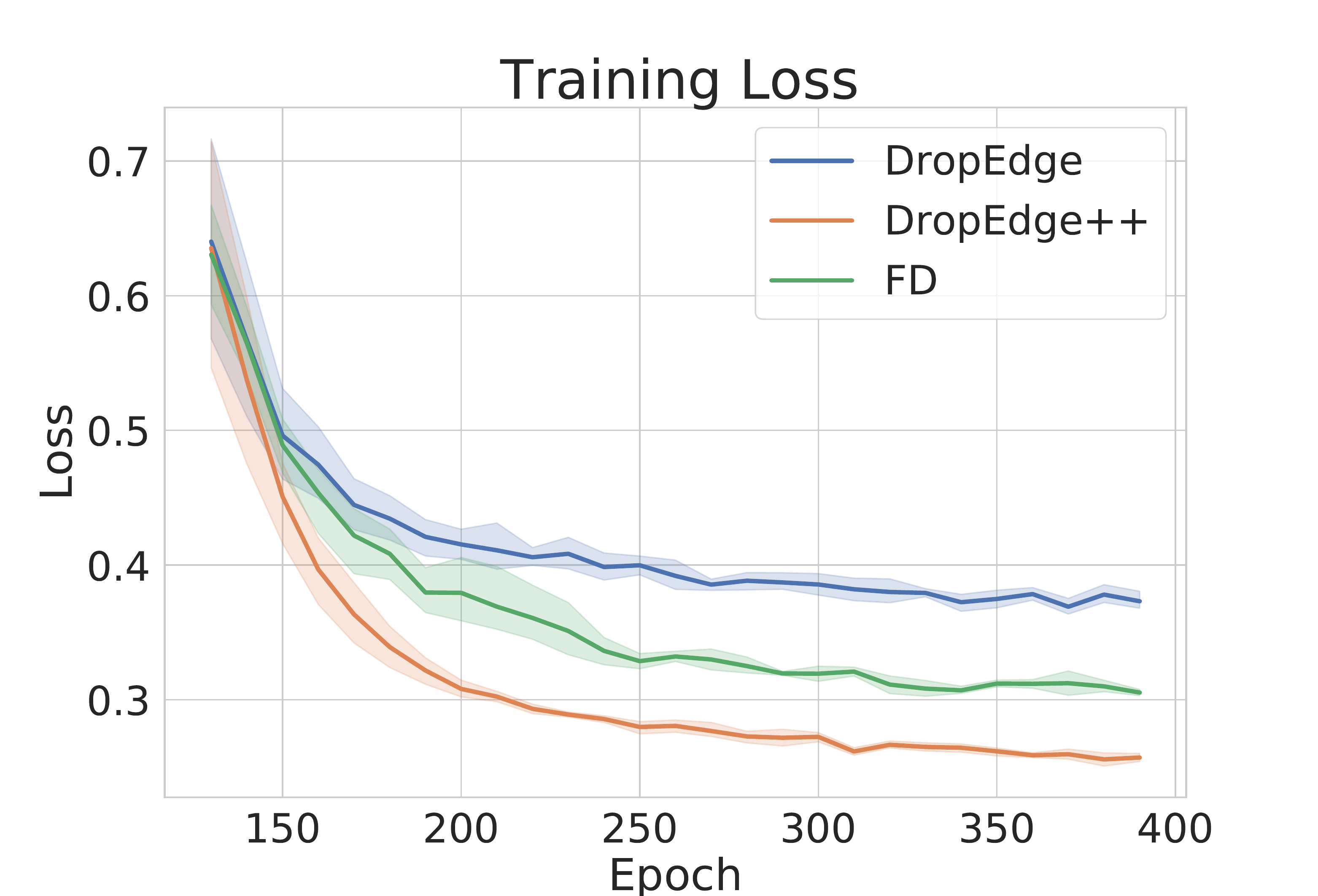}
\includegraphics[width=0.32\textwidth]{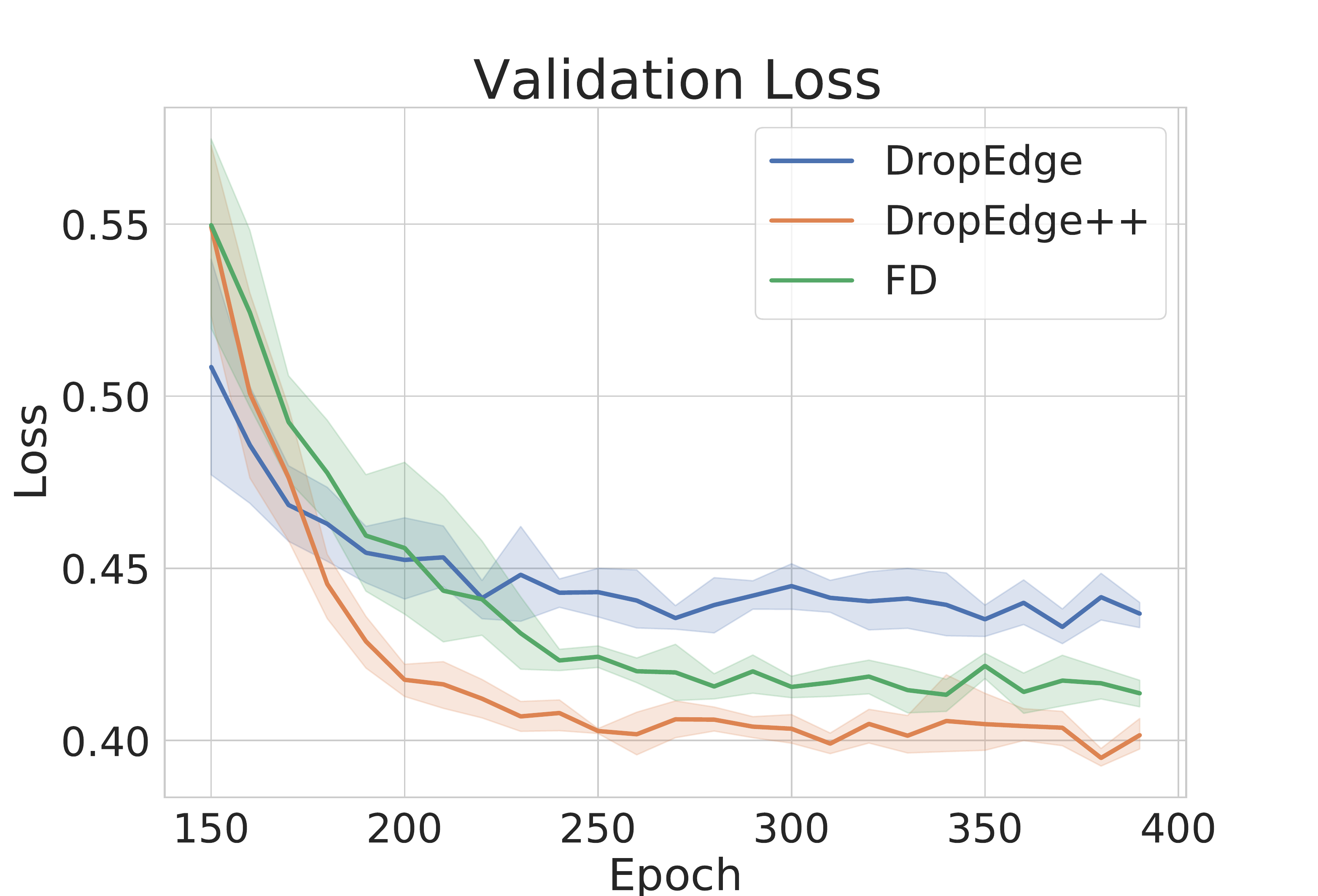}
\includegraphics[width=0.32\textwidth]{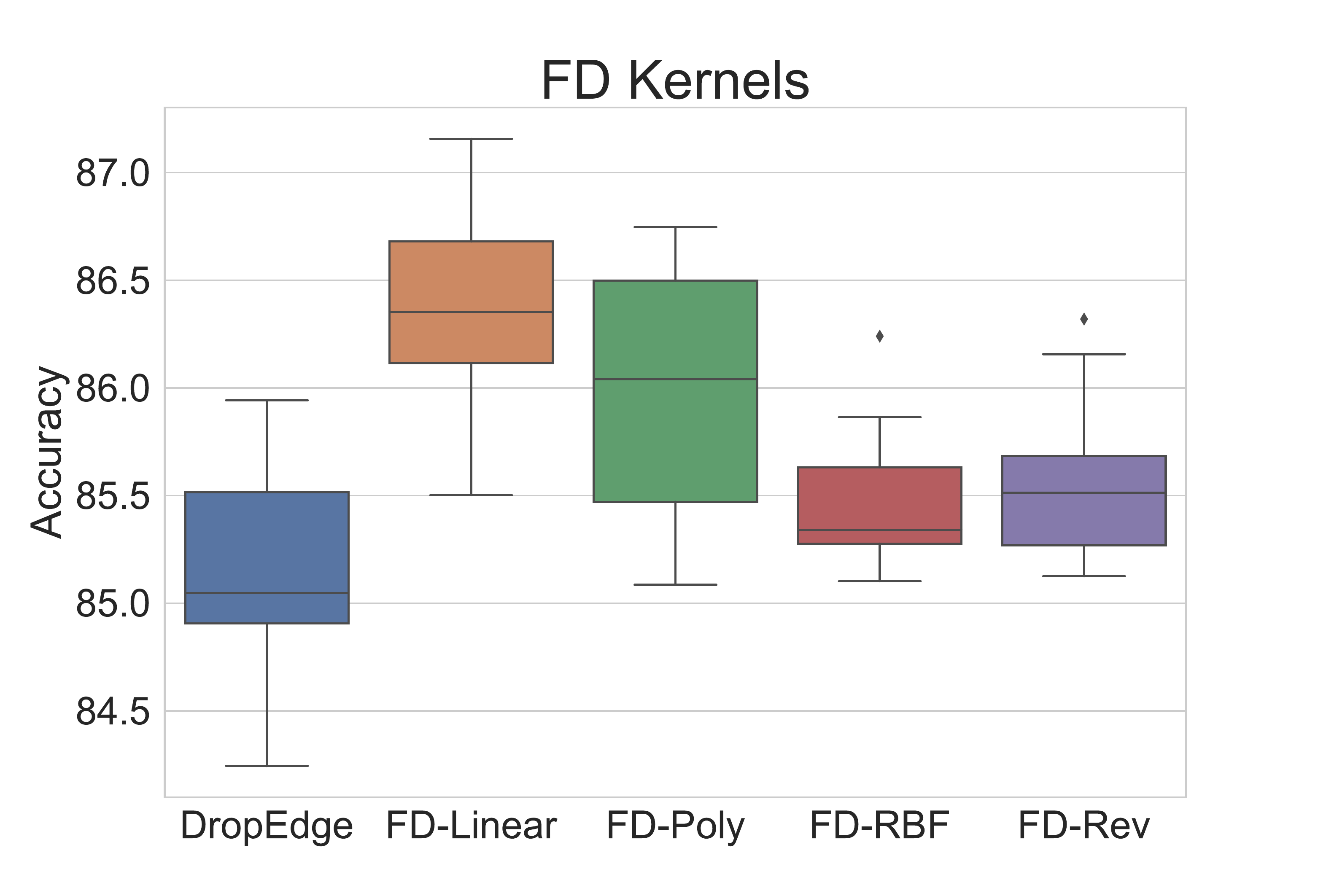}
\caption{Analysis of FD on full-supervised Cora. Left: training loss; Middle: validation loss; Right: FD with various kernels. FD is implemented with the Linear kernel by default.}
\label{fig.compare_FD}
\end{figure*}

\subsection{How does LID act compared with other variants?}
\label{sec:exp_ld}

We start by validating how the proposed metric MEN affects the extent of over-smoothing. Fig.~\ref{fig.MEN_detail_analysis} shows the correlation between MEN and the validation loss as well as the average distance of the hidden-layer outputs. The experiments are conducted on GCN with various layers on Cora dataset, with the four comparing sampling methods (LID, LDD, LI, and DropEdge). The results consistently show that MEN is negatively correlated with the average distance and positively related to the validation loss. In general, a lower MEN yields less over-smoothing given the same dropping rate, which aligns with our theory. Moreover, consistent with Theorem~\ref{th:mean}  and Theorem~\ref{th:mean-li}, LID produces the lowest MEN and validation loss among all compared candidates. Although it is non-trivial to compare $\text{MEN}_{\text{lid}}$ and $\text{MEN}_{\text{li}}$ theoretically, in practice we find that $\text{MEN}_{\text{lid}} \leq \text{MEN}_{\text{li}}$ holds based on our observation. 

We then give a detailed comparison in terms of the training and validation curves. Fig.~\ref{fig.compare_layer} displays the training curves and validation curves for DropEdge, LDD, LID, and LI on a 6-layer GCN under the same average dropping rate. Clearly, DropEdge gains inferior performance, indicating that employing shared adjacency matrices in different layers is not a good choice in this case. Particularly, LID achieves lower training and validation losses than LDD and LI, which verifies its better capability of preventing over-smoothing. 

\begin{table}[t!]
  \centering
  \setlength\tabcolsep{4pt}
  \caption{Testing accuracy (\%) on different backbones (Semi-supervised). ND: NoDrop, DE: DropEdge, DE+: DropEdge++.}
    \begin{tabular}{c|c|c|cccccc}
    \hline
    \multicolumn{3}{c|}{Layer}  & 2 & 4 & 8 & 16 & 32 & 64\bigstrut[t]\\
    \hline
    \multirow{12}[5]{*}{Cora} & \multirow{3}[2]{*}{GCN} & ND & 81.1  &80.4   &69.5  &64.9  & 60.3 & 28.7 \bigstrut[t] \\
          &       & DE & 82.8  & 82.0 &75.8  &75.7  &62.5 &49.5   \\
          &       & DE+ & \textbf{82.9} & \underline{\textbf{83.1}}  & \textbf{80.0} & \textbf{77.3} & \textbf{71.7} & \textbf{55.2}  \bigstrut[t]\\
\cline{2-9}          & \multirow{3}[2]{*}{ResGCN} & ND & -     & 78.8 & 75.6 & 72.2 &76.6 &61.1 \bigstrut[t]\\
          &       & DE & -     & \underline{\textbf{83.3}} &82.8  &82.7  &81.1 & 78.9  \\
          &       & DE+ & -     & 83.2 & \textbf{83.2} & \textbf{82.7} &\textbf{82.8} & \textbf{82.6} \bigstrut[t]\\
\cline{2-9}          & \multirow{3}[2]{*}{JKNet} & ND &  -    & 80.2 & 80.7  & 80.2& 81.1& 71.5   \bigstrut[t]\\
          &       & DE & -     & 83.3 & 82.6 & 83.0& 82.5&83.2   \\
          &       & DE+ & -     & \textbf{83.4} & \textbf{83.4}  & \textbf{83.6}  & \textbf{83.7} & \underline{\textbf{83.8}}  \bigstrut[t]\\
\cline{2-9}          & \multirow{3}[2]{*}{APPNP*} & ND &  -    & 83.2 & 82.9  & 82.7& 83.3& 82.2  \bigstrut[t]\\
          &       & DE & -     & 82.5 & 82.3 & 82.9 & 83.0 & 82.4  \\
          &       & DE+ & -     & \textbf{84.6} & \underline{\textbf{85.0}}  & \textbf{84.8}  & \textbf{84.7} & \textbf{84.9}  \bigstrut[t]\\
    \hline
    \multirow{12}[5]{*}{Citeseer} & \multirow{3}[2]{*}{GCN} & ND & 70.8 & 67.6 & 30.2  & 18.3 &25.0 &20.0  \bigstrut[t]\\
          &       & DE & 72.3  & 70.6  & 61.4  & 57.2 & 41.6& 34.4  \\
          &       & DE+ & \underline{\textbf{72.7}}  & \textbf{72.3} & \textbf{68.9}  & \textbf{64.2} & \textbf{50.1} & \textbf{37.0}  \bigstrut[t]\\
\cline{2-9}          & \multirow{3}[2]{*}{ResGCN} & ND & -     & 70.5  & 65.0 & 66.5& 62.6& 22.1  \bigstrut[t]\\
          &       & DE & -     & 72.2  & 71.6  & 70.1 &70.0 &65.1  \\
          &       & DE+ & -     & \textbf{72.9}  & \textbf{73.1}  & \textbf{73.3} & \textbf{73.0}& \underline{\textbf{73.6}}  \bigstrut[t]\\
\cline{2-9}          & \multirow{3}[2]{*}{JKNet} & ND & -     & 68.7 & 67.7 & 69.8 & 68.2& 63.4\bigstrut[t]\\
          &       & DE & -     & 72.6  & 71.8  & 72.6 & 70.8& 72.2  \\
          &       & DE+ & -     & \textbf{73.7} & \underline{\textbf{74.1}}  & \textbf{73.9} & \textbf{73.6}& \textbf{73.3}   \bigstrut[t]\\
\cline{2-9}          & \multirow{3}[2]{*}{APPNP*} & ND & -     & 72.3 & 71.8 & 73.1 & 72.7 & 72.1\bigstrut[t]\\
          &       & DE & -     & 72.4 & 72.5  & 72.3 & 73.5 & 72.7 \\
          &       & DE+ & -     & \textbf{74.5} & \textbf{74.8}  & \textbf{74.8} & \underline{\textbf{74.9}}& \textbf{74.7}   \bigstrut[t]\\
    \hline
    \multirow{12}[5]{*}{Pubmed} & \multirow{3}[2]{*}{GCN} & ND & 79.0  & 76.5 & 61.2 & 40.9 &22.4 & 35.3  \bigstrut[t]\\
          &       & DE & 79.6& 79.4 & 78.1  & 78.5 & 77.0&61.5  \\
          &       & DE+ & \textbf{79.8}  & \underline{\textbf{80.0}}  & \textbf{79.9}  & \textbf{78.8} & \textbf{77.5} & \textbf{74.7}    \bigstrut[t]\\
\cline{2-9}          & \multirow{3}[2]{*}{ResGCN} & ND & -     &78.6 &78.1  &75.5 &67.9 &66.9  \bigstrut[t]\\
          &       & DE & -     &78.8  &78.9   &78.0  &78.2 &76.9   \\
          &       & DE+ & -     & \textbf{80.2} & \underline{\textbf{80.6}}  & \textbf{80.1} & \textbf{79.9} & \textbf{79.9}   \bigstrut[t]\\
\cline{2-9}          & \multirow{3}[2]{*}{JKNet} & ND &  -    & 78.0  & 78.1  &72.6  &72.4 &74.5  \bigstrut[t]\\
          &       & DE & -     & 78.7  &78.7   & 79.1 &79.2 &78.9  \\
          &       & DE+ & -     & \textbf{80.0}  & \textbf{80.2}  & \textbf{80.3} & \underline{\textbf{80.4}} & \textbf{80.0}   \bigstrut[t] \\
\cline{2-9}          & \multirow{3}[2]{*}{APPNP*} & ND & -     & 79.5 & 79.9 & 79.8 & 80.2 & 80.3\bigstrut[t]\\
          &       & DE & -     & 80.1  & 80.4  &80.3 &80.0& 80.5  \\
          &       & DE+ & -     & \textbf{80.9} & \textbf{81.0}  & \textbf{81.3} & \textbf{81.6}& \underline{\textbf{81.8}}   \bigstrut[t]\\
    \hline
    \end{tabular}%
  \label{tab:semi_table}%
\end{table}%

\begin{table}[htbp]
  \centering
  \setlength\tabcolsep{4pt}
  \caption{Testing accuracy (\%) comparison on different backbones (Full-supervised). ND: NoDrop, DE: DropEdge, DE+: DropEdge++.}
    \begin{tabular}{c|c|c|cccccc}
    \hline
    \multicolumn{3}{c|}{Layer}  & 2 & 4 & 8 & 16 & 32 & 64 \bigstrut[t]\\
    \hline
    \multirow{12}[5]{*}{Cora} & \multirow{3}[2]{*}{GCN} & ND & 86.1  & 85.5  & 78.8  & 82.1  & 71.6 & 52.0 \bigstrut[t] \\
          &       & DE & 86.5  & 87.6  & 85.8  & 84.3 & 74.6& 53.2 \\
          &       & DE+ & \underline{\textbf{88.3}}  & \textbf{87.8}  & \textbf{86.4}  & \textbf{85.7}  & \textbf{82.9}& \textbf{74.9} \bigstrut[t]\\
\cline{2-9}          & \multirow{3}[2]{*}{ResGCN} & ND & -     & 86.0  & 85.4  & 85.3 & 85.1&79.8  \bigstrut[t]\\
          &       & DE & -     & 87.0  & 86.9  & 86.9 & 86.8& 84.8  \\
          &       & DE+ & -     &\underline{\textbf{88.5}}  & \textbf{88.4} & \textbf{88.2}  & \textbf{88.0}& \textbf{87.9}   \bigstrut[t]\\
\cline{2-9}          & \multirow{3}[2]{*}{JKNet} & ND &  -    & 86.9  & 86.7  & 86.2 & 87.1& 86.3  \bigstrut[t]\\
          &       & DE & -     & 87.7  & 87.8  & 88.0 &87.6 & 87.9  \\
          &       & DE+ & -     &\textbf{87.8}  & \textbf{88.3}  & \underline{\textbf{88.6}}  & \textbf{88.1} & \textbf{88.2}   \bigstrut[t]\\
\cline{2-9}          & \multirow{3}[2]{*}{APPNP*} & ND &  -    & 88.3  & 88.5 & 88.5 & 88.4 & 88.5  \bigstrut[t]\\
          &       & DE & -     & 87.4  & 88.4  & 88.3 &88.4 & 88.2  \\
          &       & DE+ & -     &\textbf{88.7}  & \textbf{89.0}  & \textbf{89.0}  & \underline{\textbf{89.2}} & \textbf{89.1}  \bigstrut[t]\\
    \hline
    \multirow{12}[5]{*}{Citeseer} & \multirow{3}[2]{*}{GCN} & ND & 75.9  & 76.7  & 74.6  & 65.2 & 59.2& 44.6 \bigstrut[t]\\
          &       & DE & 78.7  & 79.2  & 77.2  & 76.8 & 61.4& 45.6 \\
          &       & DE+ & \underline{\textbf{80.5}}  & \textbf{79.5}  & \textbf{78.3}  & \textbf{77.2} & \textbf{75.2} & \textbf{56.5}   \bigstrut[t]\\
\cline{2-9}          & \multirow{3}[2]{*}{ResGCN} & ND & -     & 78.9  & 77.8  & 78.2 & 74.4& 21.2 \bigstrut[t]\\
          &       & DE & -     & 78.8  & 78.8  & 79.4 & 77.9& 75.3 \\
          &       & DE+ & -     & \textbf{80.6}  & \textbf{80.8}  & \underline{\textbf{81.0}} & \textbf{80.3} & \textbf{80.2}   \bigstrut[t]\\
\cline{2-9}          & \multirow{3}[2]{*}{JKNet} & ND & -     & 79.1  & 79.2  & 78.8 & 71.7& 76.7 \bigstrut[t]\\
          &       & DE & -     & 80.2  & 80.2  & 80.1 &80.0 & 80.0 \\
          &       & DE+ & -     & \underline{\textbf{81.1}}  & \textbf{81.0}  & \textbf{80.5} & \textbf{80.2} & \textbf{80.8}   \bigstrut[t]\\
\cline{2-9}          & \multirow{3}[2]{*}{APPNP*} & ND & -     & 79.3  & 80.6  & 80.6 & 79.0 & 79.5 \bigstrut[t]\\
          &       & DE & -     & 79.5 & 80.3  & 79.9 &79.5 & 79.6\\
          &       & DE+ & -     & \textbf{80.9}  & \textbf{81.0}  & \underline{\textbf{81.2}} & \underline{\textbf{81.2}} & \textbf{81.1}   \bigstrut[t]\\
    \hline
    \multirow{12}[5]{*}{Pubmed} & \multirow{3}[2]{*}{GCN} & ND & 90.2  & 88.7  & 90.1  & 88.1 & 84.6& 79.7 \bigstrut[t]\\
          &       & DE & 91.2  & 91.3  & 90.9  & 90.3 & 86.2& 79.0 \\
          &       & DE+ & \underline{\textbf{91.8}}  & \textbf{91.0}  & \textbf{91.0}  & \textbf{90.7} & \textbf{90.5}& \textbf{83.1}   \bigstrut[t]\\
\cline{2-9}          & \multirow{3}[2]{*}{ResGCN} & ND & -     & 90.7  & 89.6  & 89.6 & 90.2& 87.9\bigstrut[t]\\
          &       & DE & -     & 90.7  & 90.5  & 91.0 & 91.1& 90.2  \\
          &       & DE+ & -     & \textbf{91.0}  & \textbf{91.6}  & \textbf{91.3} &\underline{\textbf{91.7}} & \textbf{91.2}   \bigstrut[t]\\
\cline{2-9}          & \multirow{3}[2]{*}{JKNet} & ND &  -    & 90.5  & 90.6  & 89.9 & 89.2&90.6  \bigstrut[t]\\
          &       & DE & -     & \textbf{91.3}  & 91.2  & 91.5 &91.3 &  \underline{\textbf{91.6}} \\
          &       & DE+ & -     & 91.2  & \underline{\textbf{91.6}}  & \underline{\textbf{91.6}} & \textbf{91.5} & \underline{\textbf{91.6}}  \bigstrut[t] \\
\cline{2-9}          & \multirow{3}[2]{*}{APPNP*} & ND &  -    & 90.5  & 90.3  & 90.6 & 90.5 & 90.2  \bigstrut[t]\\
          &       & DE & -     & 90.4  & 90.6  & 90.5 &90.6 &90.2 \\
          &       & DE+ & -     & \textbf{90.9}  & \underline{\textbf{91.2}}  & \textbf{91.0} & \textbf{90.9} & \textbf{90.9}  \bigstrut[t] \\
    \hline
    \end{tabular}%
  \label{tab:full_table}%
\end{table}

\begin{table*}[htbp]
  \centering
  \caption{Testing accuracy (\%) under semi-supervised settings. The number in parenthesis denotes the depth.}

    \begin{tabular}{c|c|c|c|c|c|c}
    \hline
    & Cora & Citeseer & Pubmed & Coauthor CS& Coauthor Physics & Amazon Photos\bigstrut[t]\\
    \hline
    GCN 
    & 81.5 & 71.1 & 79.0 & 89.80 & 92.80 & 90.60 \bigstrut[t]\\
    
 GAT 
 & 83.1 & 70.8 & 78.5 & 90.89 & 91.10 & 89.70 \bigstrut[t]\\
  APPNP 
  & 83.3 & 71.8 & 80.1 & 92.08 & 93.32 & 91.42  \bigstrut[t]\\
  GCNII
  & \textbf{85.5} & 73.4 & 80.2 & 91.34 & 93.43 & 91.05 \bigstrut[t] \\ 
  GCN (Drop)
  & 82.8 & 72.3 & 79.6 & 91.90 & 93.23 & 91.58 \bigstrut[t] \\
  JKNet 
  & 81.1 & 69.8 & 78.1 & 92.20 & 92.94 & 91.69  \bigstrut[t]\\
  GCN (AdaEdge) 
    & 82.3 & 69.7 & 77.4 & 90.30 & 93.00 & 91.50 \bigstrut[t]\\
    \hline\hline
    GCN w DE+ 
     & $\text{83.1}_{\pm \text{0.2(4)}}$  & $\text{72.8}_{\pm \text{0.4(2)}}$ & $\text{80.0}_{\pm \text{0.3(4)}}$ & $\text{92.96}_{\pm \text{0.04(2)}}$ &  $\textbf{\text{94.45}}_{\pm \text{0.10(4)}}$ & $\text{92.22}_{\pm \text{0.42(4)}}$ \bigstrut[t]\\
  ResGCN w DE+ 
      & $\text{83.2}_{\pm \text{0.4(4)}}$  & $\text{73.6}_{\pm \text{0.3(64)}}$     & $\text{80.6}_{\pm \text{0.5(8)}}$ & $\text{92.88}_{\pm \text{0.18(16)}}$ & $\text{94.04}_{\pm \text{0.11(16)}}$ & $\textbf{\text{93.06}}_{\pm \text{0.28(4)}}$\bigstrut[t]\\
  JKNet w DE+
    & $\text{83.8}_{\pm \text{0.4(64)}}$   & $\text{74.1}_{\pm \text{0.5(8)}}$   & $\text{80.4}_{\pm \text{0.2(32)}}$ & $\text{93.05}_{\pm \text{0.29(8)}}$  & $\text{94.21}_{\pm \text{0.16(8)}}$  & $\text{92.94}_{\pm \text{0.26(8)}}$\bigstrut[t]\\
  APPNP* w DE+ 
  & $\text{85.0}_{\pm \text{0.3(8)}}$  & $\textbf{\text{74.9}}_{\pm \text{0.4(32)}}$&  $\textbf{\text{81.8}}_{\pm \text{0.3(64)}}$ & $\textbf{\text{93.10}}_{\pm \text{0.15(8)}}$  & $\text{94.42}_{\pm \text{0.22(16)}}$  & $ \text{93.01}_{\pm \text{0.21(4)}}$ \bigstrut[t]\\
  \hline
    \end{tabular}%
  \label{tab:semi_sota}%

\end{table*}

\begin{table}[htbp]
  \centering
  \caption{Testing accuracy (\%) under full-supervised  settings. The number in parenthesis denotes the depth.}

    \begin{tabular}{c|c|c|c}
    \hline
    & Cora & Citeseer & Pubmed \bigstrut[t]\\
    \hline
    GCN 
    & 86.64 &79.34 & 90.22 \bigstrut[t]\\
    
 GAT 
 & 87.38 & 78.60 & 89.66 \bigstrut[t]\\
  FastGCN 
  & 85.00 & 77.60 & 88.00\bigstrut[t]\\
  AS-GCN 
  & 87.44 & 79.66 & 90.60 \bigstrut[t]\\
  APPNP 
  & 87.87 & 79.82 & 90.36 \bigstrut[t]\\
  GCNII
  & 88.25 & 79.92 & 90.05 \bigstrut[t]\\
  GCN (Drop)
  & 87.60 & 79.20 & 91.30 \bigstrut[t] \\
  GraphSAGE 
  & 82.20 & 71.40 & 87.10 \bigstrut[t]\\
    \hline\hline
    GCN w DE+  
     & $\text{88.30}_{\pm \text{0.16(2)}}$  & $\text{80.52}_{\pm \text{0.14(2)}}$ & $\textbf{\text{91.76}}_{\pm \text{0.16(2)}}$  \bigstrut[t]\\
   ResGCN w DE+  
      & $\text{88.48}_{\pm \text{0.32(4)}}$  & $\text{80.95}_{\pm \text{0.24(16)}}$     & $\text{91.66}_{\pm \text{0.22(32)}}$ \bigstrut[t]\\
  JKNet w DE+ 
    & $\text{88.55}_{\pm \text{0.15(16)}}$   & $\text{81.06}_{\pm \text{0.40(4)}}$   & $\text{91.57}_{\pm \text{0.31(8)}}$ \bigstrut[t]\\
   APPNP* w DE+  
   & $\textbf{\text{89.22}}_{\pm \text{0.11(32)}}$  & $\textbf{\text{81.21}}_{\pm \text{0.19(16)}}$&  $\text{91.20}_{\pm \text{0.06(8)}}$ \bigstrut[t]\\
   \hline
    \end{tabular}%
  \label{tab:full_sota}%

\end{table}

\subsection{How does FD act compared with DropEdge?}

\label{sec:FD_compare}
Theorem~\ref{th:FD} has proved that FD improves the expressivity of deep GCNs by associating the convergence subspace with the input features. To justify if this is the case in practice, we carry out FD (with the Linear kernel) and DropEdge upon 6-layer GCN under the same dropping rate. Apparently in Fig.~\ref{fig.compare_FD}, our FD sampler outperforms DropEdge in terms of both the training and validation curves. As an oracle reference, we also involve the loss of DropEdge++ that is equipped with FD and LID. DropEdge++ achieves the lowest loss than other variants, showing the compatibility of FD and LID. 
Except for the Linear kernel, we also test other kinds of kernels including the Polynomial kernel, RBF kernel for FD, which are all positively correlated with the feature-similarity. The specifications of these kernels are deferred to Appendix. To further validate our design, we also include a reverse version of FD, \emph{i.e.}, $K_{\text{Rev}} = 1 - K_{\text{Linear}}$. The results in the right sub-figure of Fig.~\ref{fig.compare_FD} indicates that all kernels are better than DropEdge, while Polynomial and Linear are most desirable in this case. We will use the Linear kernel in what follows by default. 

As a more complete ablation study, the right sub-figure in Fig.~\ref{fig.compare_layer} compares NoDrop, DropEdge++, and its two components, LID and FD. In accordance with our analysis, both FD and LID improve the performance over the generic GCN as well as DropEdge, while DropEdge++ achieves the best performance with the help of its two structure-aware samplers.

\subsection{The performance of DropEdge++ on various backbones of different layers.}
\label{sec:compare_deep}
We compare the testing accuracy of NoDrop, DropEdge and DropEdge++ with various layers (from 2 to 64) and backbones. 
More details about the design of hyper-parameters (such as $p_{\text{min}}$ and $p_{\text{max}}$) are provided in Appendix. Table~\ref{tab:semi_table} and Table~\ref{tab:full_table} display the results on Cora, Citeseer and Pubmed under semi- and full- supervised settings, respectively. We reuse the best metrics reported in~\cite{rong2020dropedge} for NoDrop and DropEdge. The results on the other three datasets are deferred to Appendix.

We have these findings:

\textbf{1.}
It is clearly shown that DropEdge++ is consistently effective in boosting the performance over NoDrop and DropEdge on all backbones with varying layers. 

\textbf{2.}
As for GCN, even equipped with DropEdge++, deep models are generally worse than the shallow counterparts under the full-supervised setting. Even so, DropEdge++, as a robust sampling technique, still enhances the performance of the 2-layer models over NoDrop remarkably; for example, the accuracy is increased from 86.1 to 88.3 on full-supervised Cora. Moreover, when deep GCN fails and delivers much worse performance, conducting DropEdge++ yields promising results over NoDrop and DropEdge, \emph{e.g.} 74.9 vs 53.2 with 64-layer GCN on full-supervised Cora.

\textbf{3.}
When using ResGCN, JKNet or APPNP* as the backbone, DropEdge++ generally holds steady with the increase of network depth and nearly gains the best number when the depth is beyond 4. For instance, when equipped with DropEdge++, the 8-layer JKNet achieves 74.1 and 32-layer APPNP* achieves 74.9 on semi-supervised Citeseer.

\subsection{Comparison with SOTAs}
We choose the best performance for each backbone with DropEdge++, and contrast them with State of the Arts (SOTAs) in Table \ref{tab:semi_sota} and Table \ref{tab:full_sota}, including GCN, GCN with DropEdge, GAT, AS-GCN, FastGCN, APPNP, GraphSAGE, GCNII, and AdaEdge~\cite{chen2019measuring}. We reuse the results of SOTAs reported in ~\cite{rong2020dropedge, chen2020simple} on Cora, Citeseer and Pubmed, and those on the other three datasets from~\cite{chen2019measuring,klicpera2019diffusion,chang2020spectral}. The means and standard deviations are acquired through 10 runs.
Obviously, our DropEdge++ obtains enhancement against SOTAs with an average improvement of more than 1\% on nearly all the benchmarks, which is significant given the challenge of these benchmarks. Although the performance on semi-supervised Cora is worse than GCNII possibly due to limited training labels, DropEdge++ outperforms all SOTAs on full-supervised Cora. Moreover, the best accuracy for most models with DropEdge++ is attained under the depth beyond 2, which again verifies the impact of DropEdge++ on formulating deep networks.

\begin{table}[htbp]
  \centering
  \caption{Average time consumed per epoch on Cora (in seconds).}

    \begin{tabular}{cccc}
    \hline
        & NoDrop  & DropEdge  & DropEdge++ \bigstrut[t]\\
      \hline
    Sampling   & 0.0015     & 0.0022  & 0.0043 \bigstrut[t]  \\
    Training   & 0.0125     & 0.0115   & 0.0117 \bigstrut[t]  \\
    Total   & 0.0140     & 0.0137   & 0.0160  \bigstrut[t] \\
    \hline
    \end{tabular}%
    \label{tab:time}
\end{table}

\subsection{Time Complexity} 
\label{sec:time}
Although we adopt a layer-wise sampling approach, the time consumption of DropEdge++ is not heavily burdened compared with DropEdge, since the kernel matrix in FD is pre-computed before training. As displayed in Table~\ref{tab:time} with GCN on Cora dataset, it only occupies a small percentage of the overall running time for sampling, making our DropEdge++ competitively efficient as DropEdge.

\subsection{Mean-Edge-Number}
\label{app:men}

\begin{figure}[t!]
\centering
\includegraphics[width=0.24\textwidth]{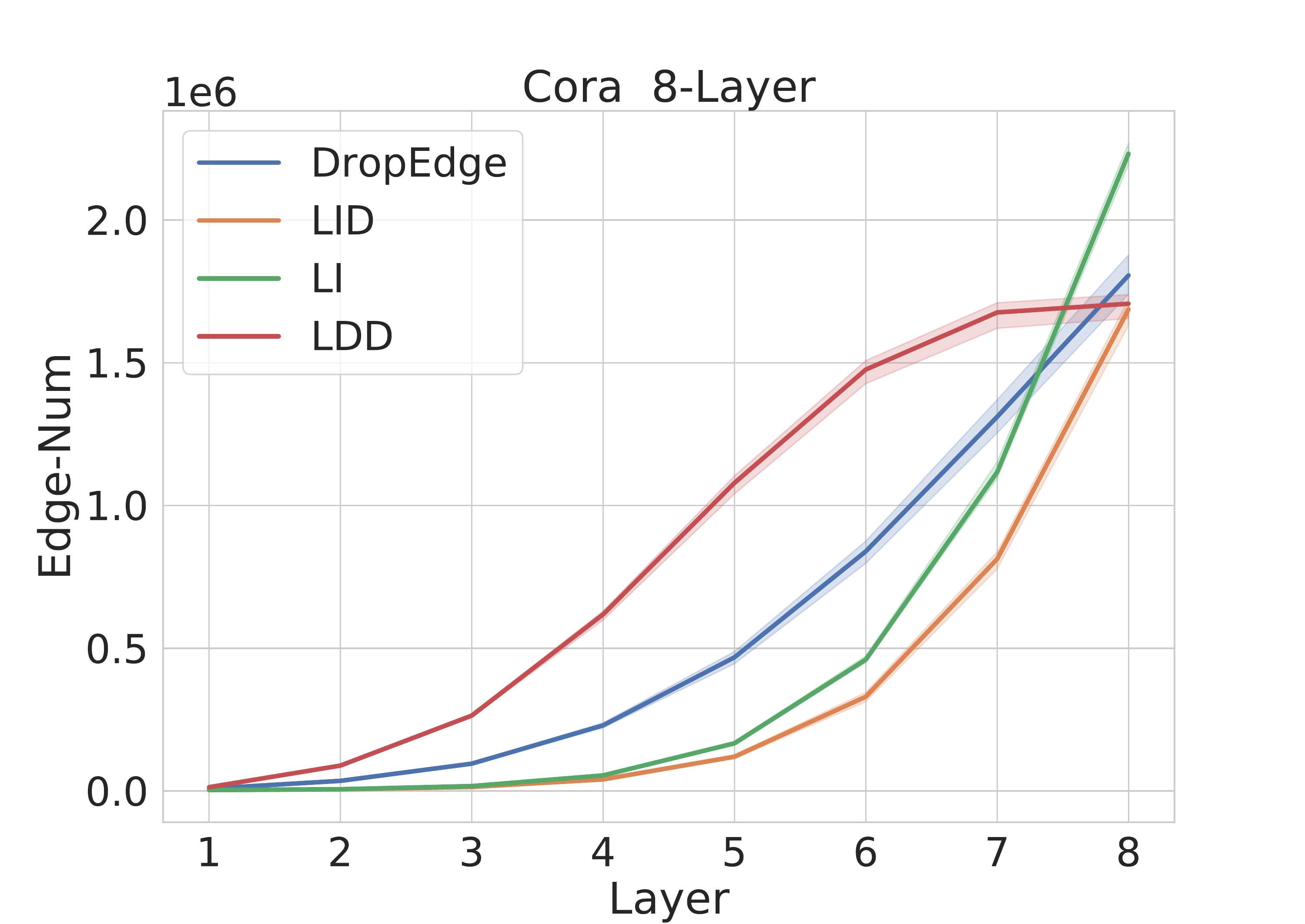}
\includegraphics[width=0.24\textwidth]{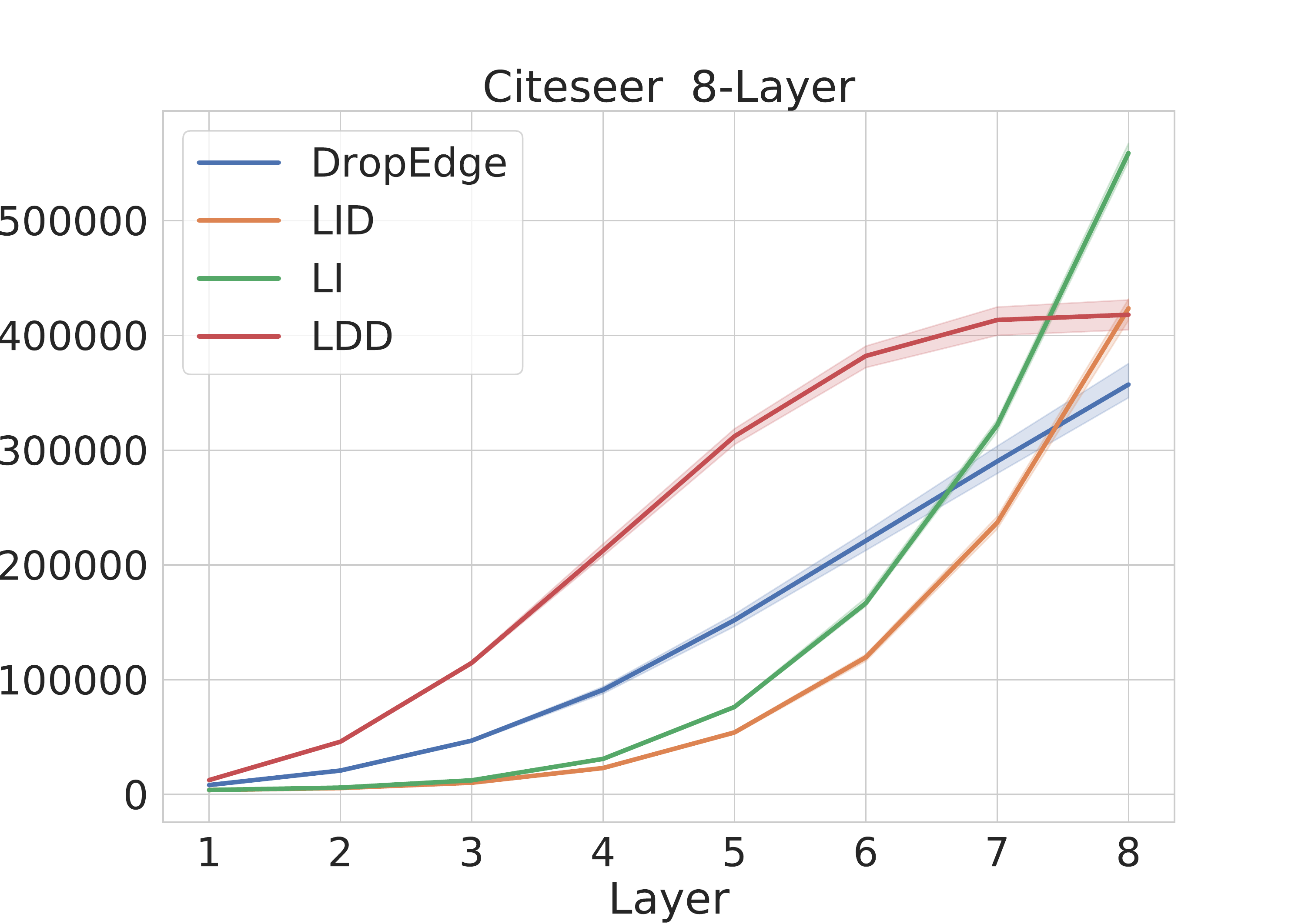}
\caption{ The edge-number ($|\prod_{i=0}^{l}\mA_{\text{drop}}^{(i)}|$) comparison on Cora and Citeseer. }
\label{fig:men_cora}
\end{figure}

\begin{table}[t!]
  \centering
  \caption{The value of MEN on Cora and Citeseer with varying layers and sampling methods (10 runs). }
    \begin{tabular}{c|c|cccccc}
    \hline
    \multicolumn{2}{c|}{Layer}        & 4  & 8  & 16\\
    \hline
    \multirow{4}[2]{*}{Cora} & DropEdge & $9.0\times10^4$ &  $6.0\times10^5$ & $1.8\times10^6$ \bigstrut[t]\\
          & LID   & $\mathbf{5.7\times10^4}$  & $\mathbf{3.8\times10^5}$ &  $\mathbf{1.3\times10^6}$ \\
          & LI    & $6.1\times10^4$ & $5.1\times10^5$ & $2.0\times10^6$ \\
          & LDD   & $1.1\times10^5$ & $8.7\times10^5$ & $3.2\times10^6$ \bigstrut[t]\\
    \hline
    \multirow{4}[2]{*}{Citeseer} & DropEdge & $4.2\times10^4$ &  $1.5\times10^5$ &  $4.2\times10^5$ \bigstrut[t]\\
          & LID   & $\mathbf{2.9\times10^4}$  & $\mathbf{1.1\times10^5}$ & $\mathbf{3.5\times10^5}$ \\
          & LI    & $3.1\times10^4$  & $1.5\times10^5$  & $6.8\times10^5$ \\
          & LDD   & $5.1\times10^4$  & $2.4\times10^5$  & $9.7\times10^5$ \bigstrut[t]\\
    \hline
    \end{tabular}%
  \label{tab:men_values}%
\end{table}%

Fig.~\ref{fig:men_cora} displays the metric, edge-number ($|\prod_{i=0}^{l}\mA_{\text{drop}}^{(i)}|$), as the vertical axis and the number of layers as the horizontal axis on various layers on Cora and Citeseer, respectively, with $p_{\text{min}} = 0.05, p_{\text{max}} = 1.0, p = (p_{\text{min}} + p_{\text{max}}) / 2$. Colored area indicates the standard deviation over 10 runs. Table~\ref{tab:men_values} shows the values of MEN on Cora and Citeseer with different number of layers and sampling methods. Clearly, our LID sampler is very effective in reducing the MEN and is beneficial to relieving over-smoothing consequently.

\section{Conclusion}
In this paper, we investigate over-smoothing on GCNs and develop two effective samplers, \emph{i.e.} the layer-increasingly-dependent sampler and the feature-dependent sampler. The layer-increasingly-dependent sampler is able to decrease the metric, mean-edge-number (MEN), and thus further reduces over-smoothing while the feature-dependent sampler further improves the expressivity of deep GCNs. Our overall framework DropEdge++ integrates these two structure-aware samplers and is demonstrated to be effective on various backbones of both deep and shallow models towards both full- and semi- supervised node classification. 

\textbf{Limitations and future works}
The feature-dependent sampler requires to specify a kernel function to compute the similarity between connected nodes. There could also be extra hyper-parameters to optimize depending on the type of the kernel, e.g., the bandwidth and scale of the RBF kernel. As future work, it would be interesting to incorporate deep kernel learning to parameterize the kernel in a unified manner, and also to enable back-propagation of the gradient into the kernel parameters via techniques like Gumbel-Softmax~\cite{jang2017categorical} for discrete sampling. Besides, instead of computing the feature similarity purely based on the input node feature, it would be more flexible to derive the similarity from some refined features acquired or distilled via feature engineering or graph representation learning frameworks.

\section*{Acknowledgment}

This work was jointly supported by the following projects: the Scientific Innovation 2030 Major Project for New Generation of AI under Grant NO. 2020AAA0107300, Ministry of Science and Technology of the People's Republic of China; the National Natural Science Foundation of China (No.62006137); Tencent AI Lab Rhino-Bird Visiting Scholars Program (VS2022TEG001); Beijing Outstanding Young
Scientist Program (No. BJJWZYJH012019100020098); Scientific Research Fund Project of Renmin University of China (Start-up Fund Project for New Teachers, No. 23XNKJ19).

\appendices

\begin{table}[t!]
  \centering
  \caption{The propagation models.}
  \label{tab:prop}
    \begin{tabular}{l|l}
    \hline
    Backbone & Propagation \bigstrut[t]\\
    \hline
    GCN   & $\mH^{(l+1)} = \sigma\left(\hat{\mA}\mH^{(l)}\mW^{(l)}\right)$ \bigstrut[t]\\
    \hline
    ResGCN & $\mH^{(l+1)} = \sigma\left(\hat{\mA}\mH^{(l)}\mW^{(l)}\right) + \mH^{(l)} $ \bigstrut[t]\\
    \hline
    JKNet &  \makecell[l]{$\mH^{(l+1)} = \sigma\left(\hat{\mA}\mH^{(l)}\mW^{(l)}\right) $, \\
    $\mH = ||_{l=0}^L \mH^{(l)}$}
    \bigstrut[t]\\
    \hline
    APPNP* & \makecell[l]{$\mZ ^{(0)}  = f_{\theta}(\mX)$, \\ $\mZ ^ {(l+1)}  = \sigma((1 - \alpha)\hat{\mA}^{(l)}\mZ^{(l)} + \alpha\mZ^{(0)})$} \bigstrut[t]\\
    \hline
    \end{tabular}%
\end{table}%

\begin{table}[t!]
    \centering
    \caption{Hyper-parameter descriptions.}
    \begin{tabular}{l|l}
\hline
Hyper-parameter & Description \bigstrut[t]\\
\hline
lr    & the learning rate \bigstrut[t]\\
$L_2$ & $L_2$ regularization \bigstrut[t]\\
$p_{\text{min}}, p_{\text{max}}'$     & the min/max sampling rate \bigstrut[t]\\
hidden & the dimension of hidden units \bigstrut[t]\\
$\alpha$ & the coefficient in APPNP* \bigstrut[t]\\
dropout & the dropout rate \bigstrut[t]\\
normalization & the normalization method in Table~\ref{tab:normalization} \bigstrut[t]\\
Loop &  self-feature modeling \bigstrut[t]\\
BN & batch normalization \bigstrut[t]\\
\hline
\end{tabular}%
\label{tab:hyper}

\end{table}

\begin{table}[t!]
    \centering
    \caption{The applied normalization methods.}
\begin{tabular}{l|l}
\hline
Abbr. & Name and Mathematical Form \bigstrut[t] \\
\hline
NA & \makecell[l]{Normlized Adjacency \\ $\mD ^{-1/2}\mA\mD^{-1/2}$} \bigstrut[t]\\
\hline
FOG & \makecell[l]{First-order GCN \\ $\mI + \mD^{-1/2}\mA\mD^{-1/2}$} \bigstrut[t]\\
\hline
AN & \makecell[l]{Augmented Normalized Adjacency \\ $(\mD + \mI)^{-1/2} ( \mA + \mI ) (\mD + \mI)^{-1/2}$} \bigstrut[t]\\
\hline
ANS & \makecell[l]{Augmented Normalized  Adjacency with Self-Loop \\$\mI + (\mD + \mI)^{-1/2} (\mA + \mI) (\mD + \mI)^{-1/2}$ }\\
\hline
ARW & \makecell[l]{Augmented Random Walk\\ $(\mD + \mI)^{-1}(\mA + \mI)$} \\
\hline
\end{tabular}%
\label{tab:normalization}
\end{table}

\begin{figure*}[t!]
\centering
\includegraphics[width=0.24\textwidth]{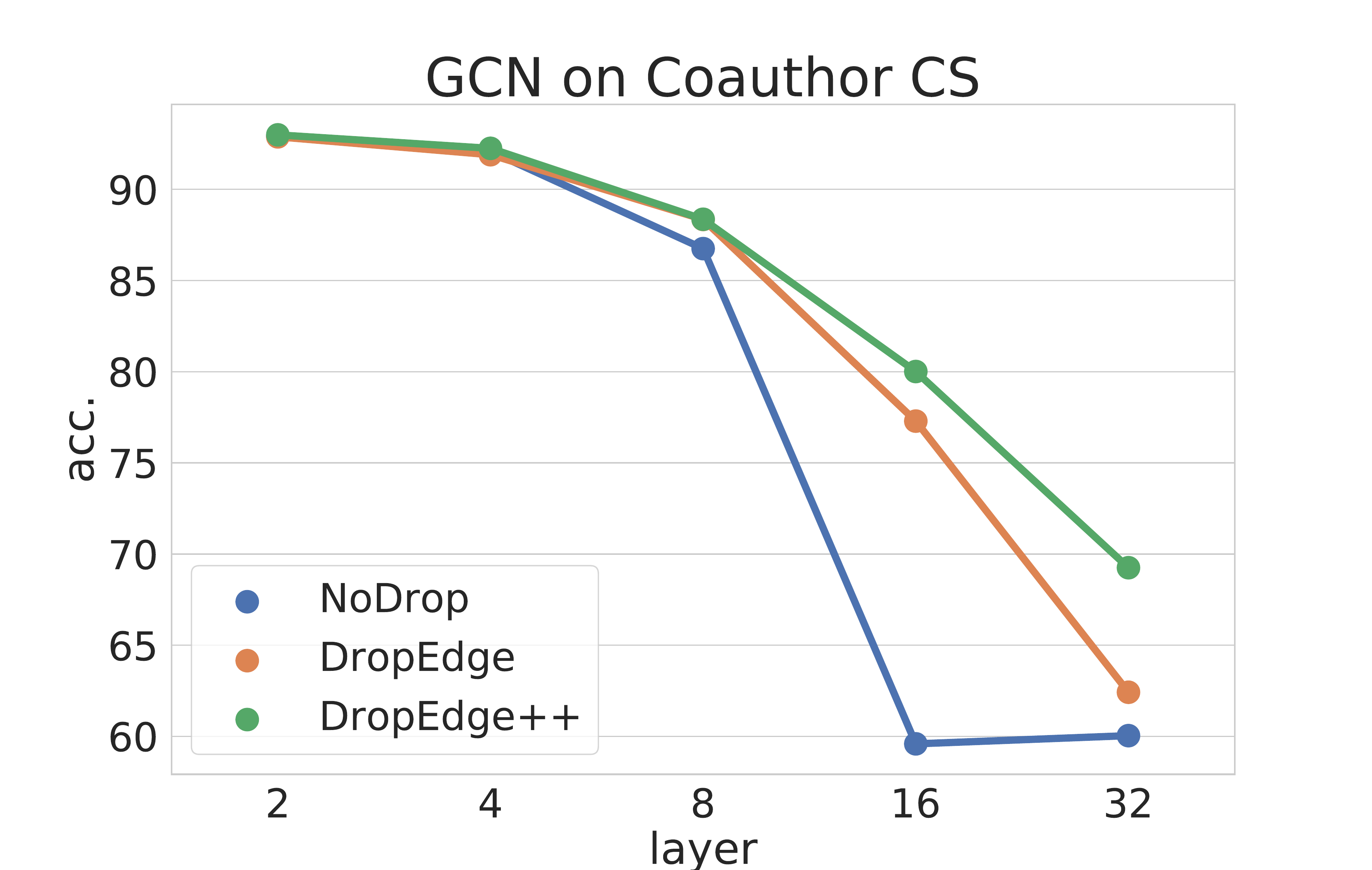}
\includegraphics[width=0.24\textwidth]{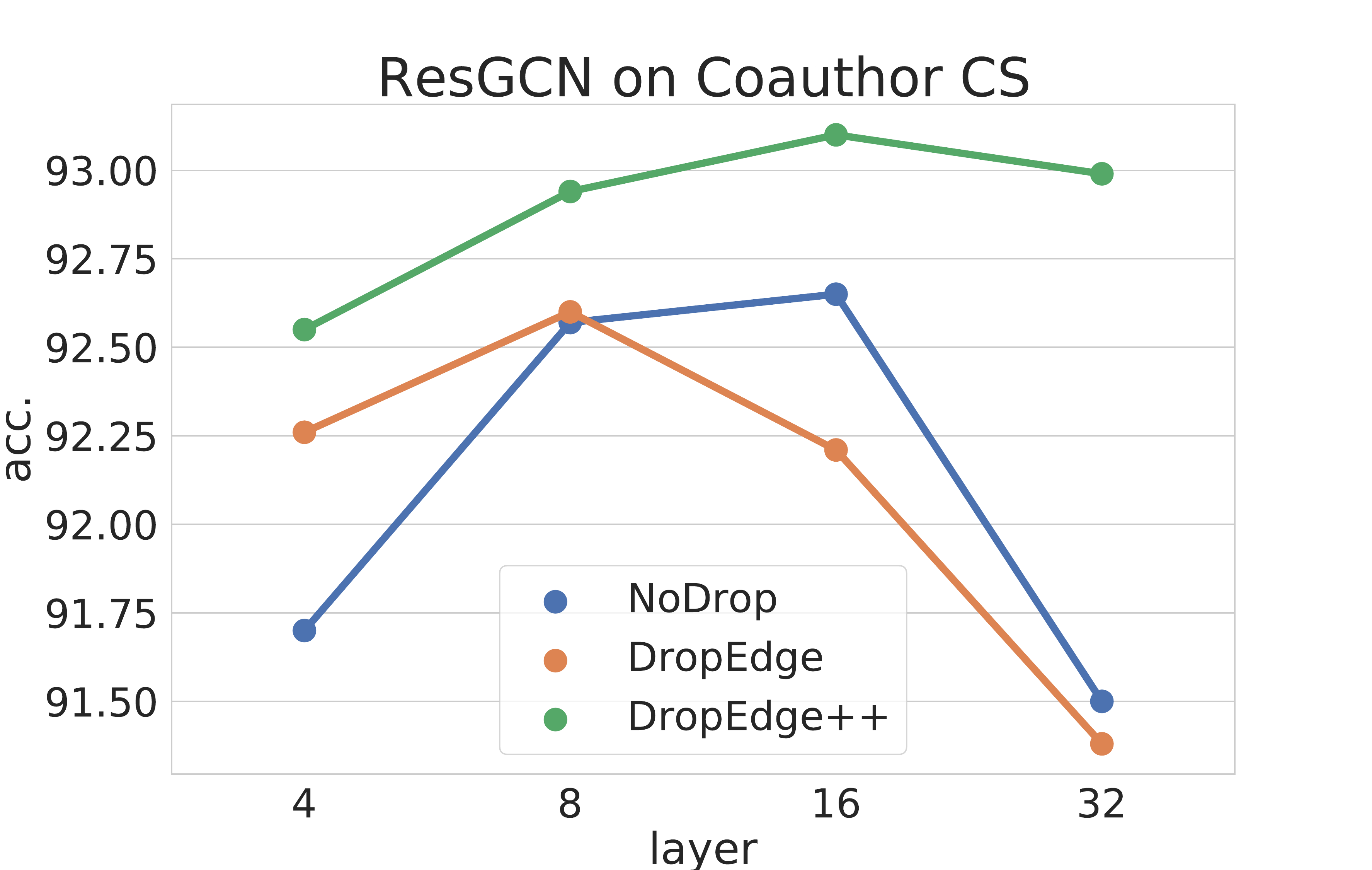}
\includegraphics[width=0.24\textwidth]{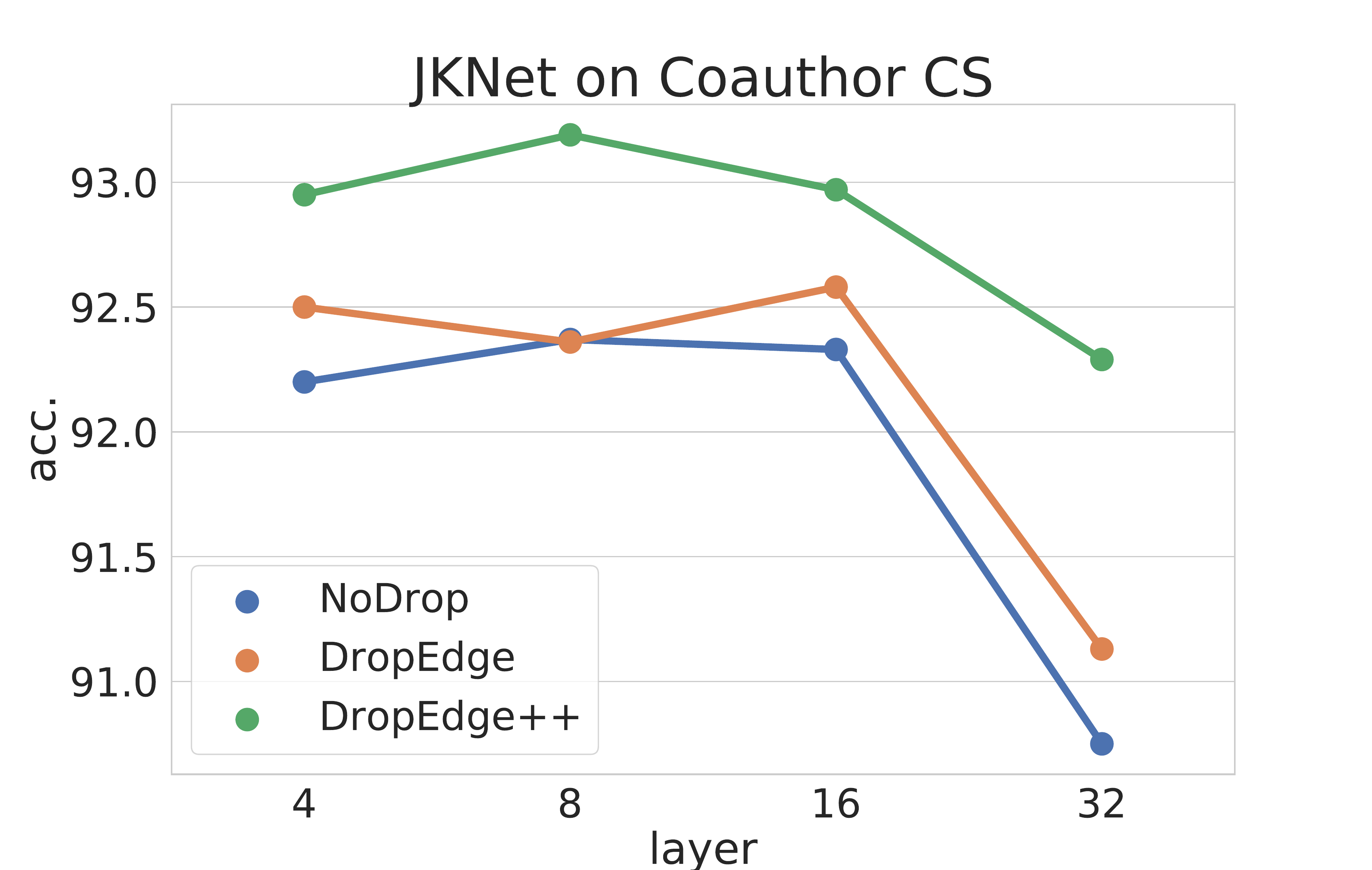}
\includegraphics[width=0.24\textwidth]{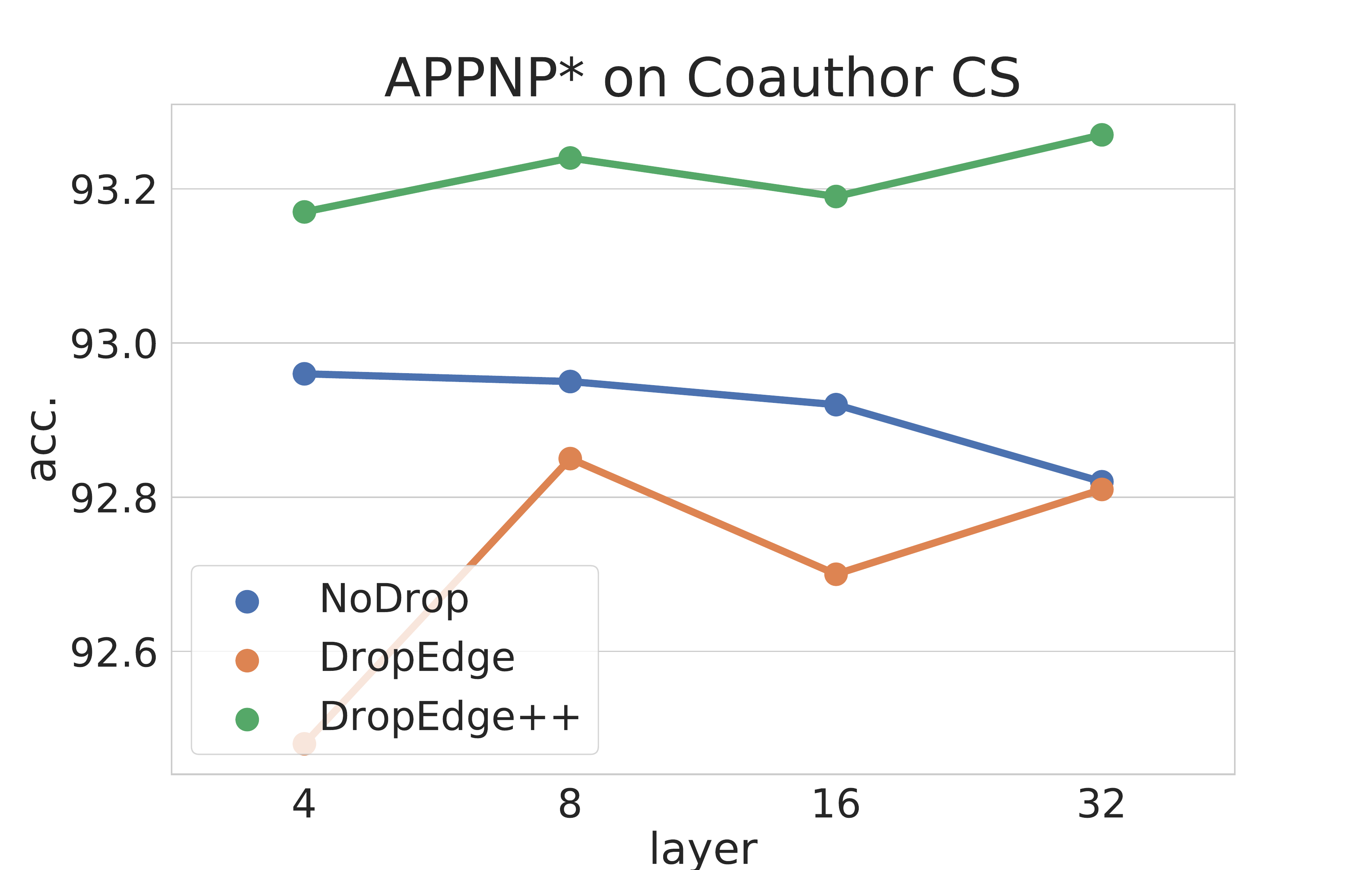}

\includegraphics[width=0.24\textwidth]{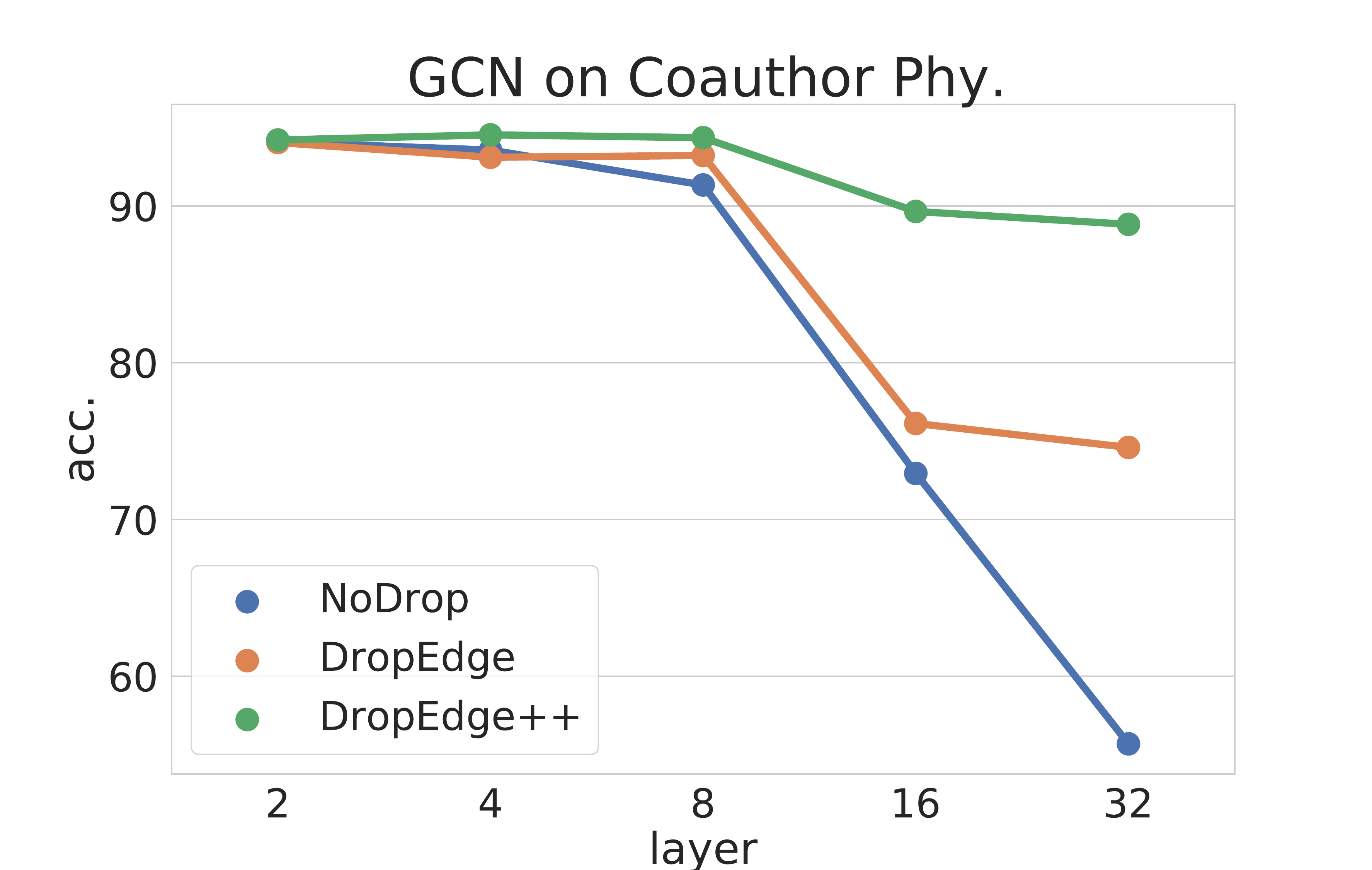}
\includegraphics[width=0.24\textwidth]{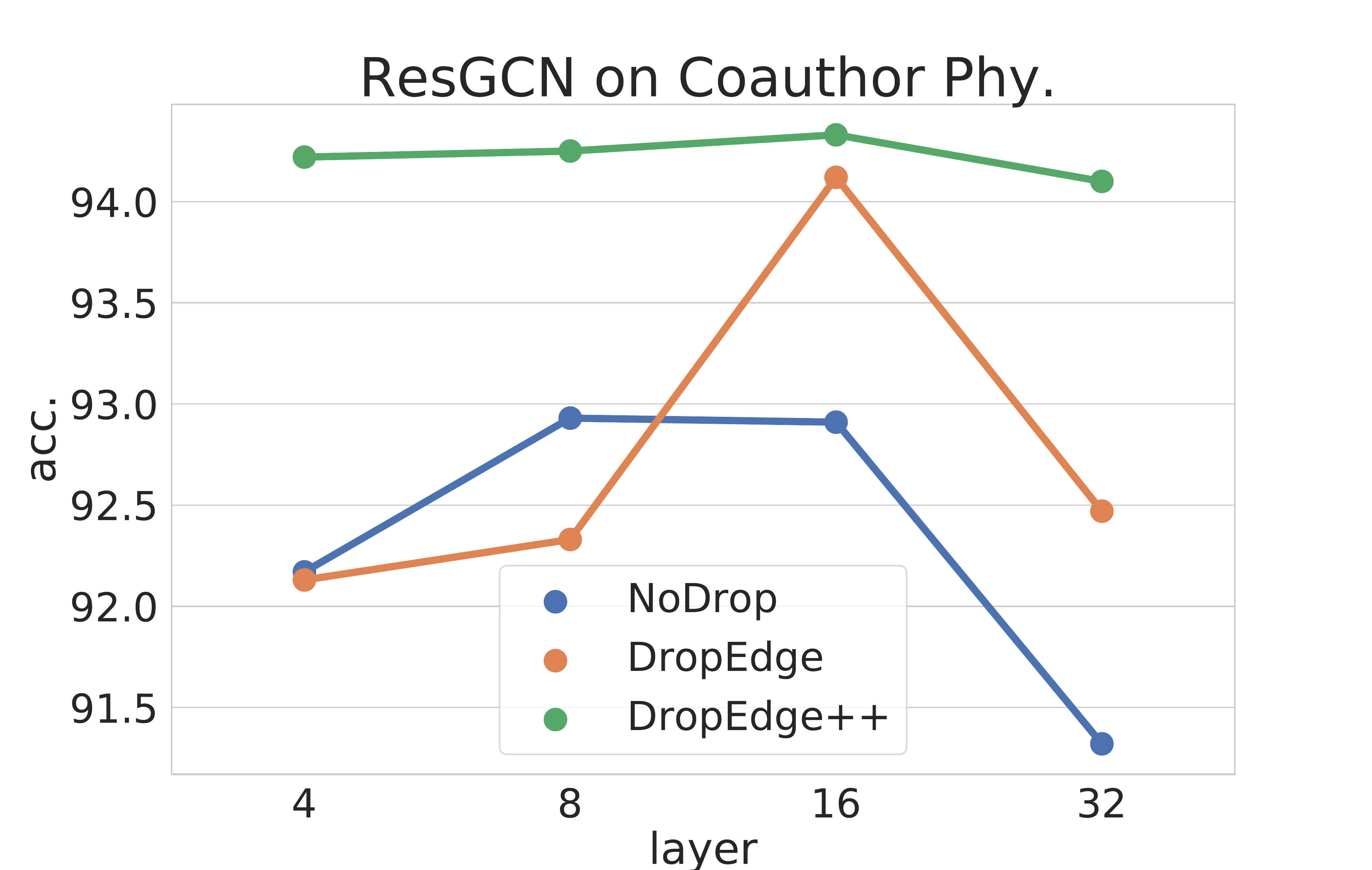}
\includegraphics[width=0.24\textwidth]{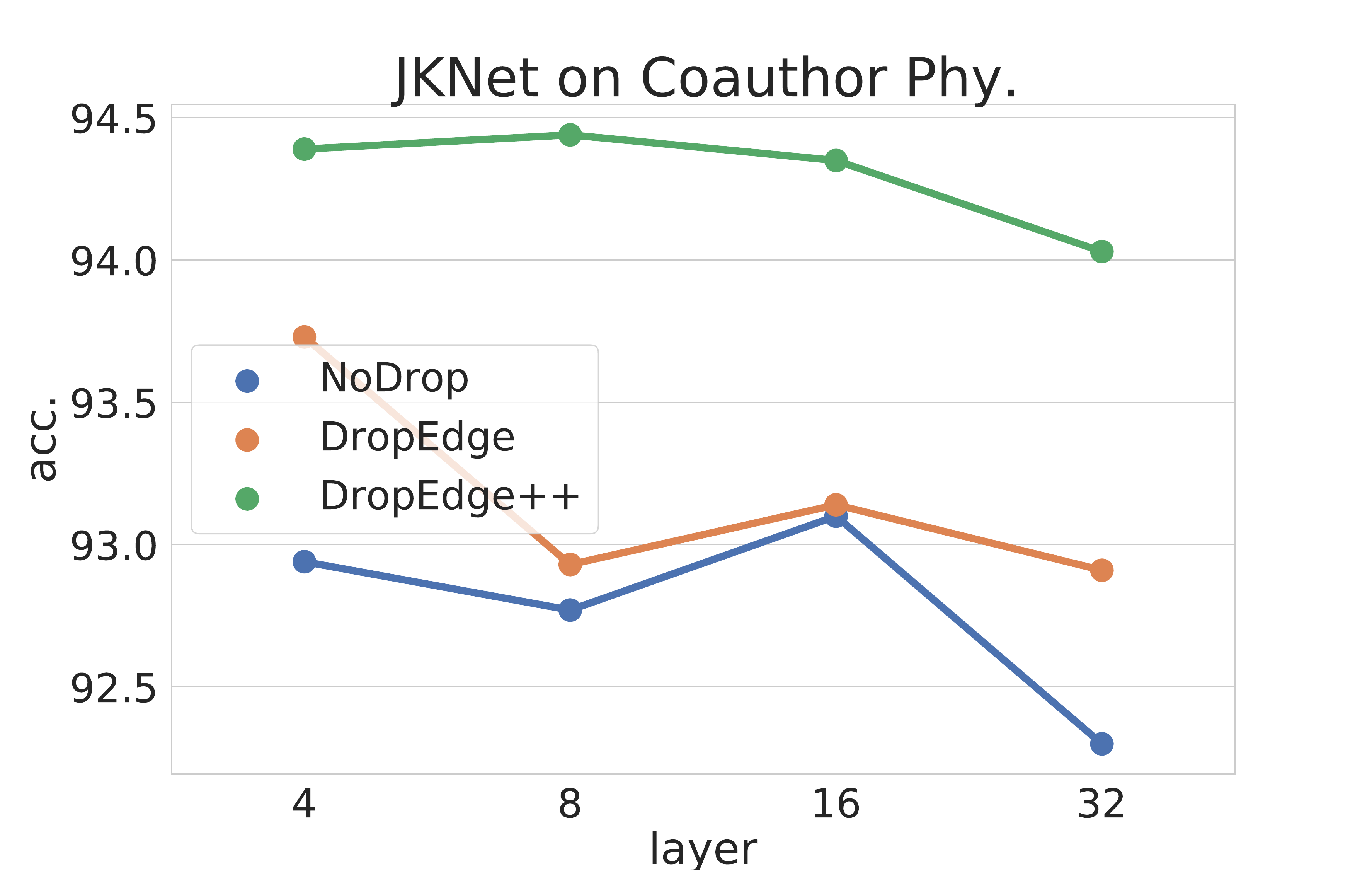}
\includegraphics[width=0.24\textwidth]{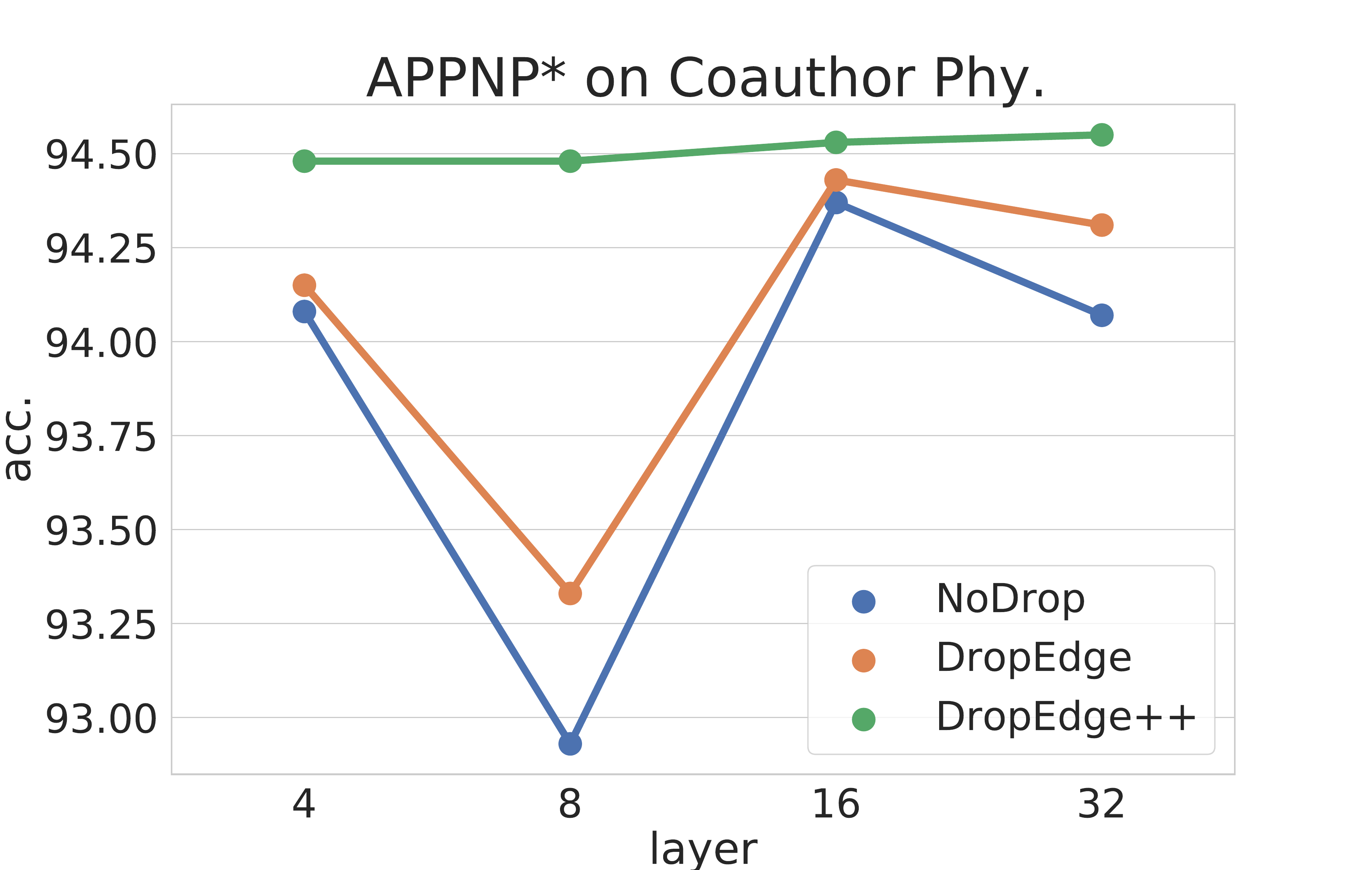}

\includegraphics[width=0.24\textwidth]{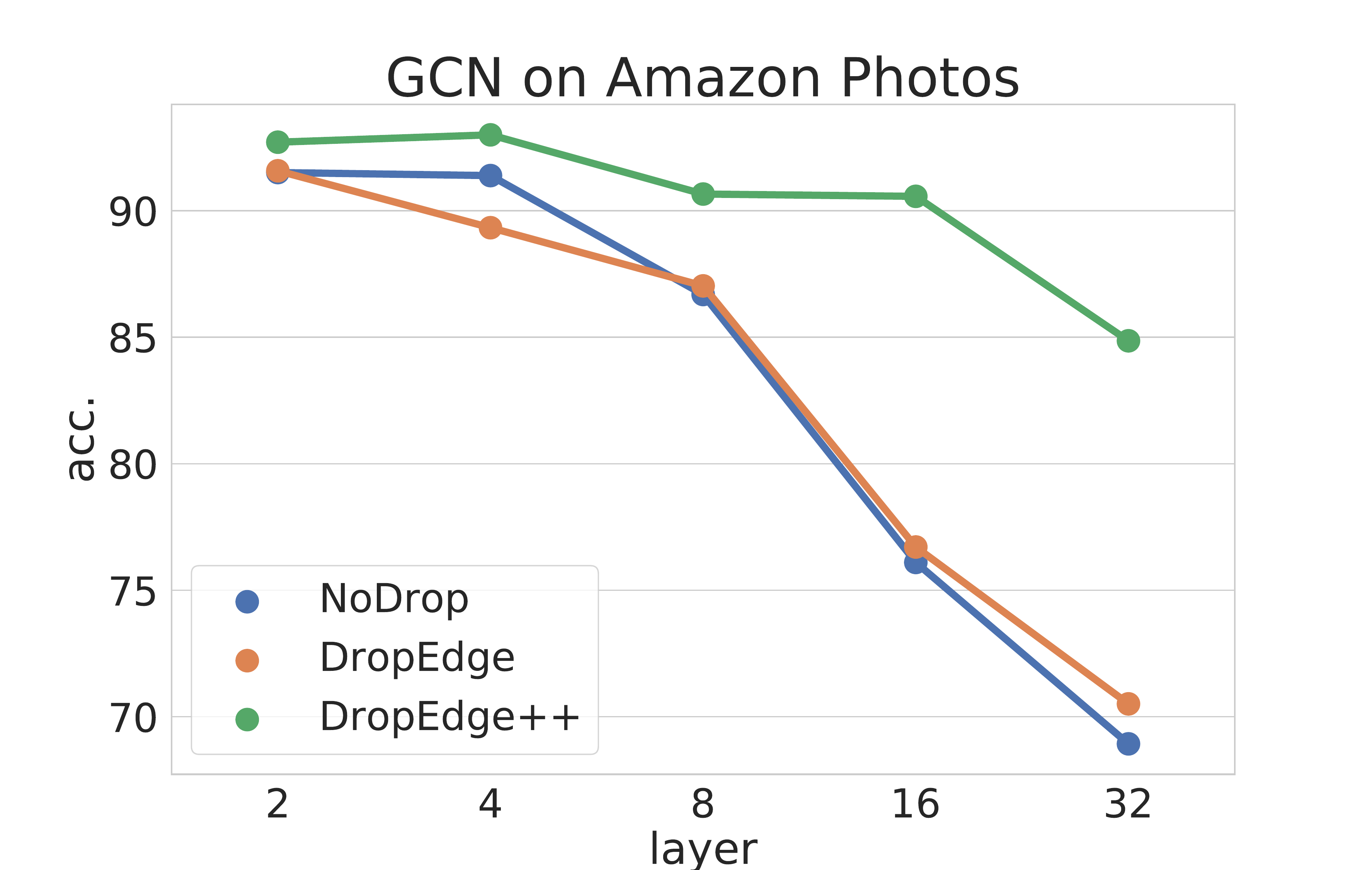}
\includegraphics[width=0.24\textwidth]{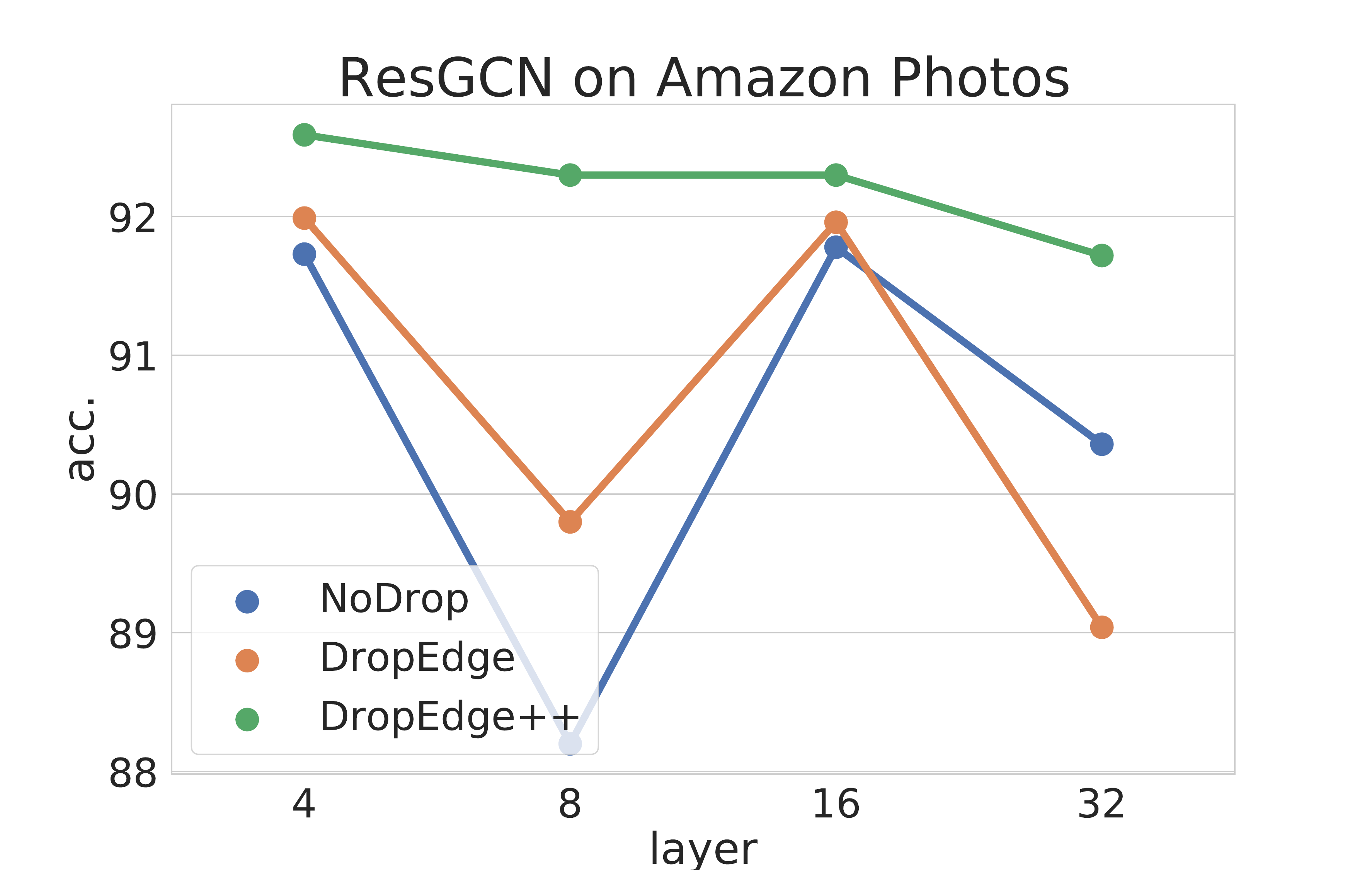}
\includegraphics[width=0.24\textwidth]{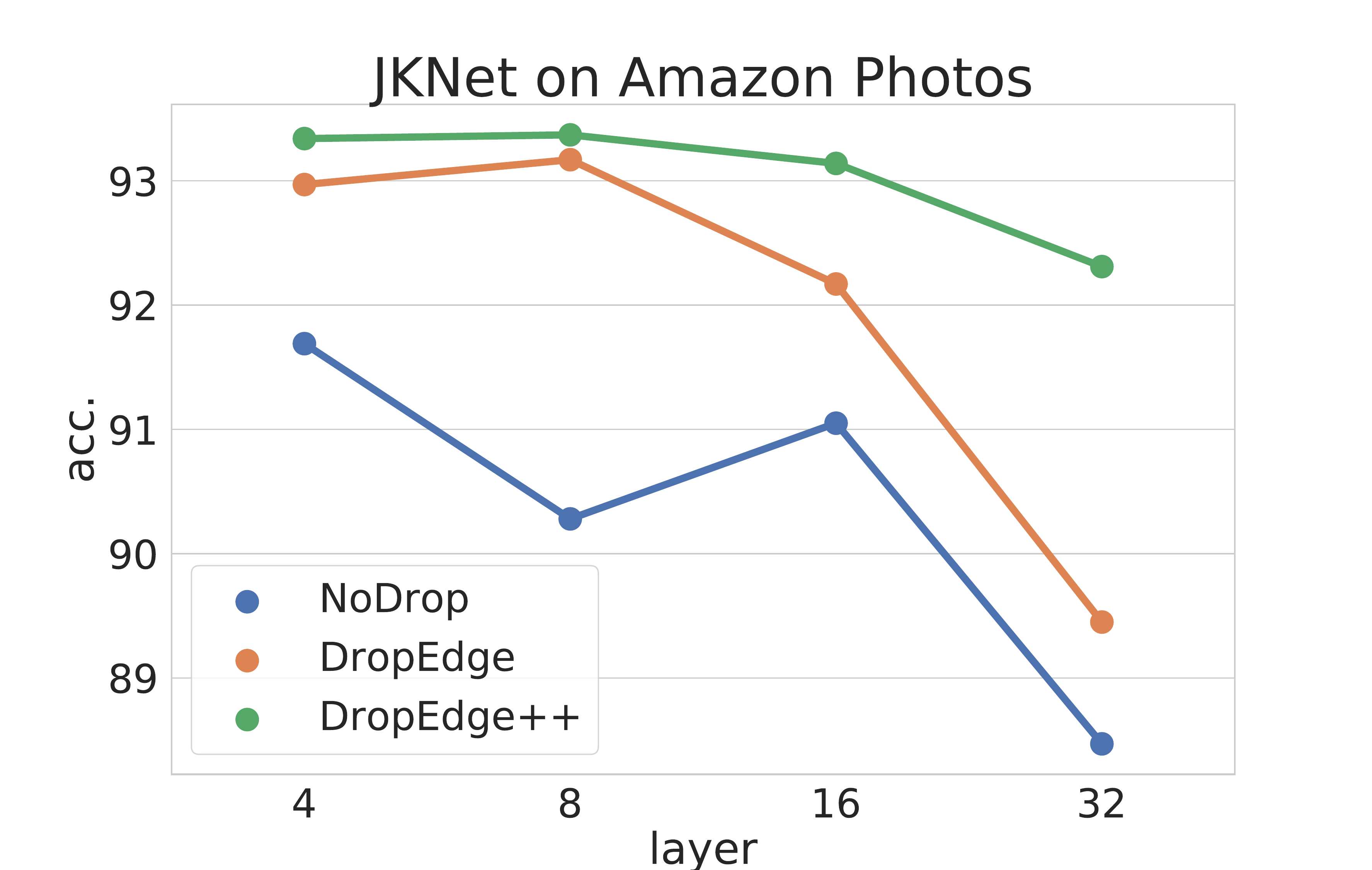}
\includegraphics[width=0.24\textwidth]{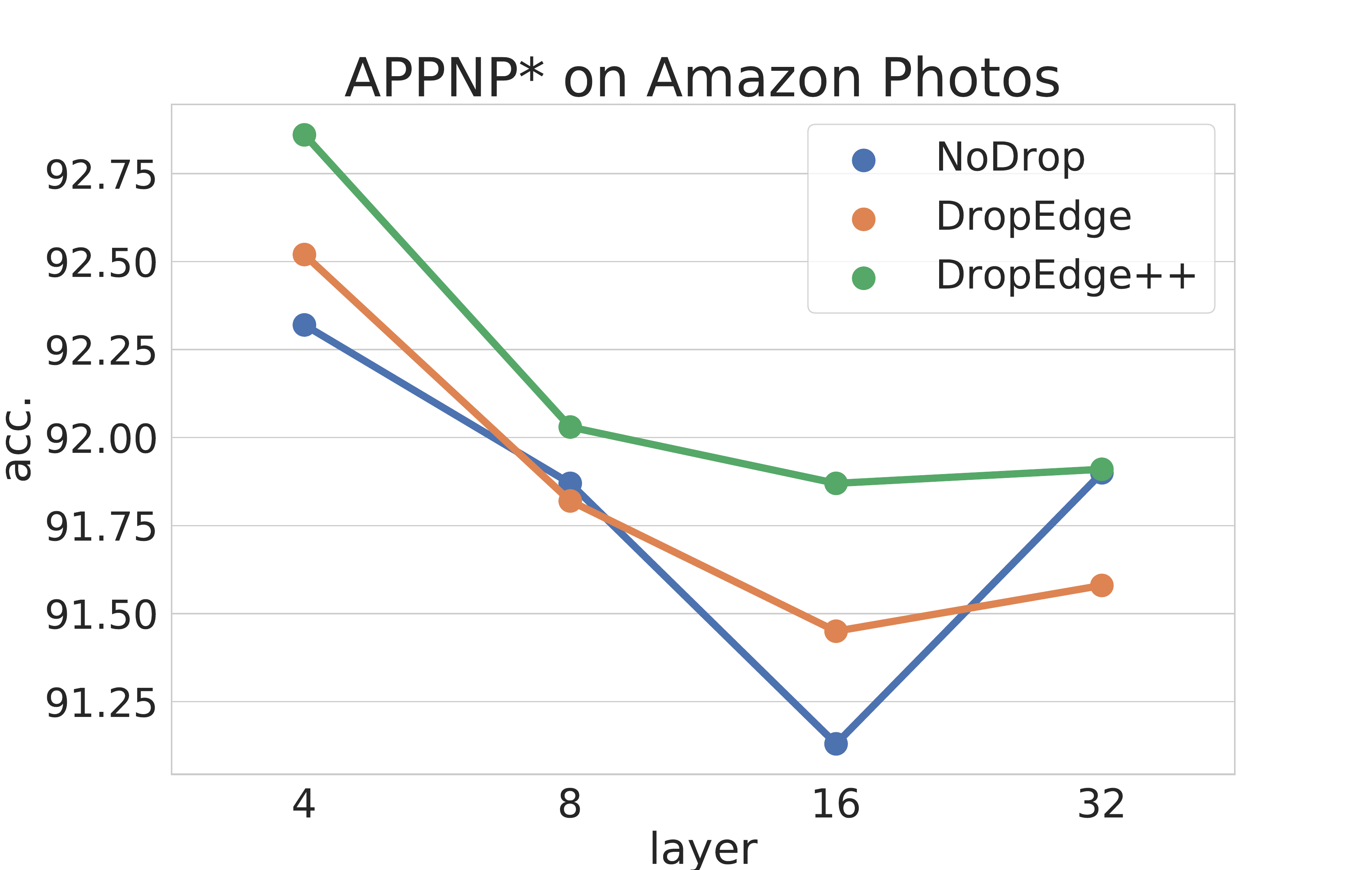}
\caption{Testing accuracy of varying models on Coauthor CS, Coauthor Phy, and Amazon Photos. }
\label{fig.coauthor}
\end{figure*}

\section{More Experiment Details and Results}
\label{sec:exp_details}

\subsection{Hyper-parameters}
\label{sec:hyperparams}
The hyper-parameters are listed in Table~\ref{tab:hyper}. Our proposed DropEdge++ implements the LID sampler, which indicates the inequality $p_{\text{max}} \geq p_{\text{min}}$ always holds. Consequently, it is intuitive to denote $p_{\text{max}}'$ as a replacement of $p_{\text{max}}$ that satisfies $p_{\text{max}}' = (p_{\text{max}} - p_{\text{min}}) / (1 - p_{\text{min}})$. With this transformation, the domain of $p_{\text{max}}$ is modified into $0\leq p_{\text{max}}'\leq 1$ naturally in the LID sampler. In our experiments, we use $p_{\text{min}}$ and $p_{\text{max}}'$ to determine the sampling rates throughout the layers, with $0 \leq p_{\text{min}} \leq 1$ and $0 \leq p_{\text{max}}' \leq 1$.

The adopted normalization methods are presented in Table~\ref{tab:normalization} together with their abbreviations. Table~\ref{tab:hyper-all} shows the values of hyper-parameters for achieving the best testing accuracy, chosen by a random search. Note that \emph{norm.} and \emph{drop.} are short for normalization, and dropout, respectively.
Experiments are conducted with $\text{layer} \in \{ \text{2, 4, 8, 16, 32, 64}\}$ for the citation networks and $\text{layer} \in \{ \text{2, 4, 8, 16, 32}\}$ for the co-author and co-purchase networks due to GPU memory limit; $ \text{hidden} \in \{\text{64, 128, 256}\}$; $L_2 \in \{\text{5e-3, 1e-3, 8e-4, 5e-4, 1e-4, 5e-5, 1e-5}\}$; $\text{dropout}\in \{\text{0.1, 0.3, 0.5, 0.8}\}$; $\alpha \in \{\text{0.1, 0.2, 0.5}\}$. As for the sampling process, $p_{\text{min}}$ and $p_{\text{max}}'$ are both selected from $\{\text{0.05, 0.1, 0.2,}\cdots\text{,1.0}\}$.

\begin{table*}[htbp] 
  \centering
  \caption{Hyper-parameters for achieving the best testing accuracy on all datasets.}
    \begin{tabular}{c|c|ccccccccccc}
    \hline
    \multicolumn{1}{c}{} &        & Layer  & hidden  & lr    & $L_2$    & $p_{\text{min}}$     & $p_{\text{max}}'$     & norm. & drop. & BN    & Loop & $\alpha$ \bigstrut[t]\\
    \hline
    \multirow{4}[2]{*}{Cora (Full)} & GCN    & 2     & 256   & 0.004 & 5e-4  & 0.1   & 0.4   & ARW   & 0.3   &       &  &\bigstrut[t]\\
          & ResGCN   & 4     & 256   & 0.004 & 5e-3  & 0.1   & 0.3   & ARW   & 0.1   &       & \checkmark &\bigstrut[t]\\
          & JKNet    & 8     & 256   & 0.006 & 8e-4  & 0.5   & 0.6   & ARW   & 0.8   &       & \checkmark &\bigstrut[t]\\
         & APPNP* & 32 & 64 & 0.004 & 5e-5 & 0.5 & 1.0 & ARW & 0.5 &  & &0.2\bigstrut[t]\\
    \hline
    \multirow{4}[2]{*}{Citeseer (Full)} & GCN     & 2     & 256   & 0.006 & 1e-3  & 0.4   & 0.9   & AN    & 0.5   &       & \checkmark &\bigstrut[t]\\
          & ResGCN   & 16    & 256   & 0.010 & 5e-5  & 0.5   & 0.7   & ANS   & 0.3   &       & \checkmark &\bigstrut[t]\\
          & JKNet      & 4     & 128   & 0.009 & 5e-4  & 0.2   & 0.4   & ANS   & 0.8   &       & \checkmark &\bigstrut[t]\\
                   & APPNP* & 16 & 128 & 0.010 & 5e-6 & 0.6 & 0.0 & ARW & 0.5 &  & &0.5\bigstrut[t]\\
    \hline
    \multirow{4}[2]{*}{Pubmed (Full)} & GCN     & 2     & 256   & 0.001 & 1e-3  & 0.1   & 0.4   & ANS   & 0.8   & \checkmark     & \checkmark &\bigstrut[t]\\
          & ResGCN   & 8     & 256   & 0.003 & 1e-3  & 0.1   & 0.4   & AN    & 0.8   & \checkmark     & \checkmark &\bigstrut[t]\\
          & JKNet     & 8     & 128   & 0.005 & 1e-4  & 0.6   & 0.9   & ARW   & 0.8   & \checkmark     & \checkmark &\bigstrut[t]\\
                   & APPNP* & 8 & 128 & 0.008 & 5e-5 & 0.9 & 0.6 & ARW & 0.3 &  & &0.5\bigstrut[t]\\
    \hline
    \multirow{4}[2]{*}{Cora (Semi)} & GCN    & 4     & 256   & 0.009 & 1e-3  & 0.8   & 0.2   & ANS  & 0.8   &       &  &\bigstrut[t]\\
          & ResGCN   & 8      & 64  & 0.007 & 5e-5  & 0.6   & 1.0   & ARW   & 0.8  &       &  &\bigstrut[t]\\
          & JKNet    & 64     & 256   & 0.001 & 1e-3  & 0.05   & 1.0   & AN   & 0.8   &       & &\bigstrut[t]\\
                   & APPNP* & 8 & 64 & 0.010 & 8e-4 & 0.9 & 0.8 & NA & 0.5 &  & &0.1\bigstrut[t]\\
    \hline
    \multirow{4}[2]{*}{Citeseer (Semi)} & GCN      & 2     & 128  & 0.009 & 1e-3  & 0.8   & 0.8   & ANS    & 0.8   &       & & \bigstrut[t]\\
          & ResGCN    & 64    & 64   & 0.007 & 1e-3  & 0.6   & 0.6   & ARW   & 0.8   &       & & \bigstrut[t]\\
          & JKNet      & 8    & 128   & 0.007 & 1e-3  & 0.6   & 0.0   & ARW   & 0.8   &       & &\bigstrut[t]\\
                   & APPNP* & 32 & 64 & 0.010 & 8e-4 & 0.9 & 0.4 & NA & 0.5 &  & &0.1\bigstrut[t]\\
    \hline
    \multirow{4}[2]{*}{Pubmed (Semi)} & GCN     & 4    & 256   & 0.005 & 1e-3  & 0.4   & 0.5   & NA   & 0.8   &    &  &\bigstrut[t]\\
          & ResGCN   & 8     & 128  & 0.007 & 1e-3  & 0.6   & 1.0   & ARW    & 0.8   &      & &\bigstrut[t]\\
          & JKNet     & 32     & 64   & 0.007 & 1e-3  & 0.6   & 0.8   & ARW   & 0.8   &     &  &\bigstrut[t]\\
                   & APPNP* & 64 & 256 & 0.004 & 1e-4 & 0.3 & 0.2 & ANS & 0.1 &  & &0.5\bigstrut[t]\\
    \hline
    \multirow{4}[2]{*}{Coauthor CS} & GCN   & 2     & 128   & 0.006 & 1e-3  & 0.2   & 0.7   & ARW   & 0.5   &       & \checkmark &\bigstrut[t]\\
          & ResGCN    & 16    & 128   & 0.009 & 5e-4  & 0.05  & 0.2   & ARW   & 0.1   &       & \checkmark &\bigstrut[t]\\
          & JKNet   & 8     & 128   & 0.003 & 5e-5  & 0.8   & 0.2   & ANS   & 0.5   &       & \checkmark &\bigstrut[t]\\
                   & APPNP* & 8 & 64 & 0.008 & 5e-5 & 0.4 & 0.2 & AN & 0.5 &  & &0.2\bigstrut[t]\\
    \hline
    \multirow{4}[2]{*}{Coauthor Physics} & GCN   & 4     & 256   & 0.002 & 5e-5  & 0.3   & 0.3   & AN    & 0.1   &       & \checkmark &\bigstrut[t]\\
          & ResGCN   & 16    & 256   & 0.006 & 1e-5  & 0.05  & 1.0   & ANS   & 0.3   &       &  &\bigstrut[t]\\
          & JKNet     & 8     & 256   & 0.008 & 1e-5  & 0.7   & 0.8   & ANS   & 0.8   &       &  &\bigstrut[t]\\
                   & APPNP* & 16 & 64 & 0.005 & 8e-6 & 0.7 & 0.5 & AN & 0.5 &  & &0.2\bigstrut[t]\\
    \hline
    \multirow{4}[2]{*}{Amazon Photos} & GCN  & 4     & 128   & 0.006 & 1e-5  & 0.2   & 0.1   & FOG   & 0.5   &       & \checkmark &\bigstrut[t]\\
          & ResGCN   & 4     & 256   & 0.006 & 1e-4  & 0.3   & 0.8   & FOG   & 0.8   &       & \checkmark &\bigstrut[t]\\
          & JKNet    & 8     & 256   & 0.010 & 1e-5  & 0.4   & 0.3   & ARW   & 0.5   &       & \checkmark &\bigstrut[t]\\
                   & APPNP* & 4 & 64 & 0.010 & 1e-5 & 0.7 & 0.5 & ARW & 0.1 &  & &0.2\bigstrut[t]\\
    \hline
    \end{tabular}%
  \label{tab:hyper-all}%
\end{table*}%

\subsection{Propagation models}
Table~\ref{tab:prop} displays the four propagation models we adopted in our experiments, including GCN, ResGCN, JKNet, and our implemented APPNP*. Note that in Table~\ref{tab:prop}, the symbol ``$||$'' stands for concatenation.

\subsection{Kernels}
\label{sec:app_kernels}
We list the detailed mathematical forms of the kernels as follows. The linear kernel $K_{\text{Linear}}(\bm{x}_1,\bm{x}_2)=\bm{x}_1^{\mathrm{T}}\bm{x}_2$; the polynomial kernel $K_{\text{Poly}}(\bm{x}_1,\bm{x}_2)=(\bm{x}_1^{\mathrm{T}}\bm{x}_2)^2$; the radial basis (RBF) kernel $K_{\text{RBF}}(\bm{x}_1,\bm{x}_2) = \text{exp}(l||\bm{x}_1 - \bm{x}_2||^2)$ where we select $l=-6$ by validation. In our experiments, we use the Linear kernel by default.


\subsection{More results on the co-author and co-purchase datasets}
\label{sec:copurchase}
Fig.~\ref{fig.coauthor} displays the comparison on test accuracy with varying layers on Coauthor CS, Coauthor Physics, and Amazon Photos. In accordance with our analysis, DropEdge++ consistently boosts the performance of the backbones, and is very effective in formulating deep networks.

\ifCLASSOPTIONcaptionsoff
  \newpage
\fi



\bibliographystyle{IEEEtran}
\bibliography{IEEEabrv, ref}
\end{document}